\documentclass[english]{article}

\usepackage{algorithm}
\usepackage{amsfonts}
\usepackage{amsmath}
\usepackage{amsthm}
\usepackage{authblk}
\usepackage[colorlinks,linkcolor=blue,anchorcolor=blue,citecolor=blue,urlcolor=blue]{hyperref}
\usepackage[noend]{algpseudocode}
\usepackage{cleveref}
\usepackage{geometry}
\usepackage{graphicx}
\usepackage[utf8]{inputenc}
\usepackage{mathtools}
\usepackage{subcaption}
\usepackage{thmtools}
\usepackage{thm-restate}
\usepackage{tikz-cd}
\graphicspath{ {./figures/} }

\usepackage{aktnotes}

\makeatletter
\newcounter{algorithmicH}
\let\oldalgorithmic\algorithmic
\renewcommand{\algorithmic}{%
  \stepcounter{algorithmicH}
  \oldalgorithmic}
\renewcommand{\theHALG@line}{ALG@line.\thealgorithmicH.\arabic{ALG@line}}
\makeatother

\begin{document}

\title{Pareto-optimal clustering \\ with the primal deterministic information bottleneck}
    
\author[1, 2]{Andrew K. Tan}
\author[1, 2]{Max Tegmark}
\author[1, 2, 3]{Isaac L. Chuang}

\affil[1]{\footnotesize Department of Physics, Massachusetts Institute of Technology, Cambridge, Massachusetts 02139, USA}
\affil[2]{\footnotesize The NSF AI Institute for Artificial Intelligence and Fundamental Interactions, Cambridge, Massachusetts 02139, USA}
\affil[3]{\footnotesize Research Laboratory of Electronics, Massachusetts Institute of Technology, Cambridge, Massachusetts 02139, USA}

\date{}

\maketitle


\begin{abstract}
    At the heart of both lossy compression and clustering is a trade-off between the fidelity and size of the learned representation.
    Our goal is to map out and study the Pareto frontier that quantifies this trade-off.
    We focus on the optimization of the Deterministic Information Bottleneck (DIB) objective over the space of hard clusterings.
    To this end, we introduce the {\it primal} DIB problem, which we show results in a much richer frontier than its previously studied Lagrangian relaxation when optimized over discrete search spaces.
    We present an algorithm for mapping out the Pareto frontier of the primal DIB trade-off that is also applicable to other two-objective clustering problems.
    We study general properties of the Pareto frontier, and we give both analytic and numerical evidence for logarithmic sparsity of the frontier in general.
    We provide evidence that our algorithm has polynomial scaling despite the super-exponential search space,
    and additionally, we propose a modification to the algorithm that can be used where sampling noise is expected to be significant.
    Finally, we use our algorithm to map the DIB frontier of three different tasks: compressing the English alphabet, extracting informative color classes from natural images, and compressing a group theory-inspired dataset, revealing interesting features of frontier, and demonstrating how the structure of the frontier can be used for model selection with a focus on points previously hidden by the cloak of the convex~hull.
\end{abstract}


\section{Introduction}
    \label{sec:introduction}
    Many important machine learning tasks can be cast as an optimization of two objectives that are fundamentally in conflict: performance and parsimony.
    In an auto-encoder, this trade-off is between the fidelity of the reconstruction and narrowness of the bottleneck.
    In the rate-distortion setting, the quantities of interest are the distortion, as quantified by a prescribed distortion function, and the captured information.
    For clustering, the trade-off is between intra-cluster variation and the number of clusters.
    While these problems come in many flavors---with different motivations, domains, objectives, and solutions---what is common to all such multi-objective trade-offs is the existence of a Pareto frontier, representing the boundary separating feasible solutions from infeasible ones.
    In a two-objective optimization problem, this boundary is generically a one-dimensional curve in the objective plane, representing solutions to the trade-off where increasing performance along one axis necessarily decreases performance along the other.

    The shape of the frontier, at least locally, is important for model selection: prominent corners on the frontier are often more robust to changes in the inputs and therefore correspond to more desirable solutions.
    The global frontier can provide additional insights, such as giving a sense of interesting scales in the objective function.
    Structure often exists at multiple scales;
    for hierarchical clustering problems, these are the scales at which the data naturally resolve.
    Unfortunately, much of this useful structure (see Figure~\ref{fig:titular}) is inaccessible to optimizers of the more commonly studied convex relaxations of the trade-offs.
    Optimization over discrete search spaces poses a particular difficulty to convex relaxed formulations, as most points on the convex hull are not feasible solutions, and Pareto optimal solutions are masked by the hull.
    While the optimization of the Lagrangian relaxation is often sufficient for finding a point on or near the frontier, we, in contrast, seek to map out the entire frontier of the trade-off and therefore choose to tackle the primal problem directly.
    \begin{figure}[H]
        \centering
        \includegraphics[width=13.5cm]{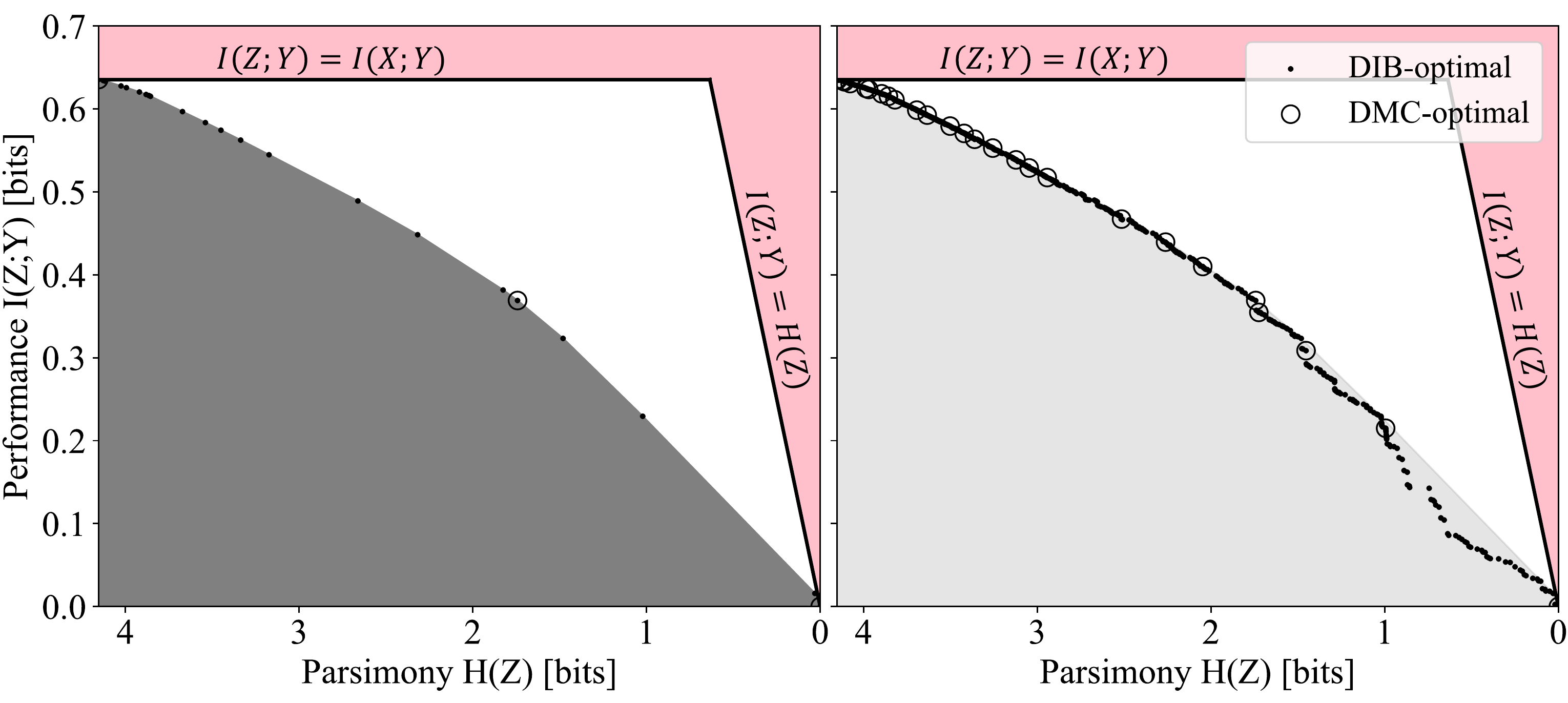}
        \caption{Comparison of the Lagrangian DIB (\textbf{left}) and primal DIB (\textbf{right}) frontiers discovered by Algorithm~\ref{alg:mapper} for the English alphabet compression task discussed in Section~\ref{sec:alphabet-dataset}.
        The shaded regions indicate the convex hull of the points found by the Pareto Mapper algorithm.
        Clusterings inside the shaded region, while Pareto optimal, are not optimal in the Lagrangian formulation.
        }
        \label{fig:titular}
    \end{figure}
    
    \begin{algorithm}[H]
        \caption{Pareto Mapper: \(\varepsilon\)-greedy agglomerative search}
        \label{alg:mapper}
        \textit{Input}: Joint distribution \(X, Y \sim p_{XY}\), and search parameter $\varepsilon$
         
        \textit{Output}: Approximate Pareto frontier \(P\)
        \begin{algorithmic}[1]
            \Procedure{Pareto\_mapper}{$p_{XY}, \varepsilon$}
                \State \textbf{Pareto set} \(P = \emptyset\) \Comment{Initialize maintained Pareto set}
                \State \textbf{Queue} \(Q = \emptyset\) \Comment{Initialize search queue}

                \State \textbf{Point} \(p = (\mathrm{x=} -\operatorname{H}(p_X), \mathrm{y=} \operatorname{I}(p_{X; Y}), \mathrm{f=} \operatorname{id})\) \Comment{Evaluate trivial clustering}
                \State \(P \leftarrow  \textsc{insert}(p, P)\) 

                \State \(Q \leftarrow  \textsc{enqueue}(\operatorname{id}, Q) \) \Comment{Start with identity clustering \(\operatorname{id}: [n] \rightarrow [n]\) where \(n = |X|\)}
                
                \While{\(Q\) is not \(\emptyset\)}
                    \State \(f = \textsc{dequeue}(Q)\) 
                    \State \(n = |\operatorname{range}(f)|\)
                    \For{ \(0 < i < j < n \) } \Comment{Loop over all pairs of output clusters of \(f\)}
                        \State \(f' = c_{i,j} \circ f\) \Comment{Merge clusters \(i, j\) output \(f\)}
                        \State \textbf{Point} \(p = \mathrm{Point}(\mathrm{x=} -\operatorname{H}(p_{f'(X)}), \mathrm{y=} \operatorname{I}(p_{f'(X); Y}), \mathrm{f=} f')\) 
                        \State \(d = \textsc{pareto\_distance}(p, P)\)

                        \State \(P \leftarrow \textsc{pareto\_add}(p, P)\) \Comment{Keep track of point and clustering in Pareto set}

                        \State{\textbf{with} probability \(e^{-d/\varepsilon}\),  \(Q \leftarrow \textsc{enqueue}(f', Q)\)} \\
                    \EndFor
                \EndWhile 
                \Return{\(P\)}
            \EndProcedure
        \end{algorithmic}
    \end{algorithm}
    
    We focus on the general problem of the deterministic encoding of a discrete domain.
    For a finite set of inputs, \(X\), which we identify with the integers \([X] \equiv \{1, \ldots, |X|\}\), we seek a mapping to a set \(Z\), where \(|Z| \le |X|\). 
    The search space is therefore the space of functions \(f: [X] \rightarrow [Z]\), which we call ``encodings'' or equivalently, ``hard clusterings'', where \(Z=f(X)\) is interpreted as the number of the cluster to which \(X\) is assigned.
    We evaluate the encodings using the Deterministic Information Bottleneck objective, but regardless of which objectives are chosen, we will refer to all two-objective optimization problems over the space of such functions \(f\) as ``clustering problems''.

    The goal of this paper is to motivate the study of the Pareto frontiers to primal clustering problems and to present a practical method for their discovery.

    \subsection{Objectives and Relation to Prior Work}
    \label{sec:objectives-and-priorwork}
    
    We focus on the task of lossy compression, which is a trade-off between retaining the salient features of a source and parsimony.
    Rate-distortion theory provides the theoretical underpinnings for studying lossy data compression of a source random variable \(X\) into a compressed representation \(Z\) \cite{CT06}.
    In this formalism, a distortion function quantifies the dissatisfaction with a given encoding, which is balanced against the complexity of the compressed representation as measured by \(I(Z; X)\).
    In the well-known Information Bottleneck (IB) problem \cite{TPB00}, the goal is to preserve information about a target random variable \(Y\) as measured by the mutual information \(I(Z; Y)\);
    the IB can be viewed as a rate-distortion problem with the Kullback–Leibler divergence, \(D_{KL}( p_{Y | X} || p_{Y | Z})\), serving as the measure of distortion.
    In recent years, a number of similar bottlenecks have also been proposed inspired by the IB problem \cite{SS17b, AF18, F20}.
    We focus on one of these bottlenecks known as the Deterministic Information Bottleneck (DIB) \cite{SS17b}.

    \subsubsection{The Deterministic Information Bottleneck}
    \label{sec:dib}
    In the DIB problem, we are given random variables \(X\) and \(Y\) with joint probability mass function (PMF) \(p_{XY}\), and we would like to maximize \(I(Z; Y)\) subject to a constraint on \(H(Z)\).
    As in \cite{SS17b}, we further restrict ourselves to the compression of discrete domains, where \(X\), \(Y\) and \(Z\) are finite, discrete random variables.
    We note that DIB-optimal encodings are deterministic encodings \(Z = f(X)\) \cite{SS17b}, and we can therefore focus on searching through the space of functions \(f: [X] \rightarrow [Z]\), justifying the interpretation of DIB as a clustering problem.
    Since the optimization is being performed over a discrete domain in this case, not all points along the frontier are achievable.
    Nonetheless, we define the Pareto frontier piecewise as the curve that takes on the minimum vertical value between any two adjacent points on the frontier.

    Formally, given \(p_{XY}\), the DIB problem seeks an encoding \(f^*: X \rightarrow Z\) such that, \(Z^* = f^*(X)\) maximizes the relevant information captured for a given entropy limit \(H^*\):
    \begin{equation}
        f_{\mathrm{primal}}^* \equiv \argmax_{f: H[f(X) \rule[1.35ex]{0pt}{1.pt}]\le H^*} I \left(Y; f(X) \rule[1.85ex]{0pt}{1pt} \right)
        \label{eq:primalopt}
    \end{equation}
    We will refer to the constrained version of the DIB problem in Equation~\eqref{eq:primalopt} as the \emph{primal} DIB problem, to differentiate it from its more commonly studied Lagrangian form \cite{SS17b}:
    \begin{equation}
        f_{\mathrm{Lagrangian}}^* \equiv \argmax_f I \left(f(X); Y \rule[1.85ex]{0pt}{1pt} \right) - \beta H \left(f(X)  \rule[1.85ex]{0pt}{1pt} \right)
        \label{eq:lagrangeopt}
    \end{equation}
    In this form, which we call the \emph{Lagrangian} DIB, a trade-off parameter \(\beta\) 
    is used instead of the entropy bound \(H^*\) to parameterize $f_*$ and quantify the importance of memorization relative to parsimony.
    The Lagrangian relaxation has the benefit of removing the non-linear constraint at the cost of optimizing a proxy to the original function, known as the DIB Lagrangian, but it comes at the cost of being unable to access points off the convex hull of the trade-off.
    We note that while we use the terminology `primal DIB' to differentiate it from its Lagrangian form, we do not study its `dual' version in this paper. 

    Many algorithms have been proposed for optimizing the IB, and more recently, the DIB objectives \cite{HWD17}.
    An iterative method that generalizes the Blahut-Arimoto algorithm was proposed alongside the IB \cite{TPB00} and DIB \cite{SS17b} algorithms.
    For the hierarchical clustering of finite, discrete random variables \(X\) and \(Y\) using the IB objective, both divisive \cite{PTL94} and agglomerative \cite{ST99} methods have been studied.
    Relationships between geometric clustering and information theoretic clustering can also be used to optimize the IB objective in certain limits \cite{BMDG05}.
    More recently, methods using neural network-based variational bounds have been proposed \cite{AFDM16}.
    However, despite the wealth of proposed methods for optimizing the (D)IB, past authors \cite{HWD17, TPB00, SS17b, AFDM16, ATMS04} has focused only on the Lagrangian form of Equation~\eqref{eq:lagrangeopt} and are therefore unable to find convex portions of the frontier.

    Frontiers of the DIB Lagrangian and primal DIB trade-offs are contrasted in Figure~\ref{fig:titular}, with the shaded gray region indicating the shroud that the optimization of the Lagrangian relaxation places on the frontier (the particular frontier presented is discussed in Section~\ref{sec:alphabet-dataset}).
    Points within the shaded region are not accessible to the Lagrangian formulation of the problem as they do not optimize the Lagrangian.
    We also note that while the determinicity of solutions is a consequence of optimizing the Lagrangian DIB \cite{SS17b}, the convex regions of the primal DIB frontier is known to contain soft clusterings \cite{KTV18, TW19}.
    In our work, the restriction to hard clusterings can be seen as a part of the problem statement.
    Finally, we adopt the convention of flipping the horizontal axis as in \cite{TW19} which more closely matches the usual interpretation of a Pareto frontier where points further up and to the right are more desirable.

    \subsubsection{Discrete Memoryless Channels}
    A closely related trade-off is that between \(I(Z; Y)\) and the number of clusters \(|Z|\), which has been extensively studied in the literature on the compression of discrete memoryless channels (DMCs) \cite{KY14, ZK16, HWD17}.
    In Figure~\ref{fig:titular} and the other frontier plots presented in Section~\ref{sec:examples}, the DMC optimal points are plotted as open circles.
    The DIB and DMC trade-offs are similar enough that they are sometimes referred to interchangeably \cite{HWD17}: some previous proposals for solutions to the IB \cite{ST99} are better described as solutions to the DMC trade-off.
    We would like to make this distinction explicit, as we seek to demonstrate the richness of the DIB frontier over that of the DMC frontier.
    
    \subsubsection{Pareto Front Learning}
    In recent work by Navon et al. \cite{NSFC20}, the authors define the problem of Pareto Front Learning (PFL) as the task of discovering the Pareto frontier in a multi-objective optimization problem, allowing for a model to be selected from the frontier at runtime.
    Recent hypernetwork-based approaches to PFL \cite{NSFC20, LYZK20} are promising being both scalable and in principle capable of discovering the entirety of the primal frontier.
    Although we use a different approach, our work can be seen as an extension to the existing methods for PFL to discrete search spaces.
    Our Pareto Mapper algorithm performs PFL for the task of hard clustering, and our analysis provides evidence for the tractability of the PFL in this setting.
    
    We also note similarities to the problems of Multi-Task Learning and Multi-Objective Optimization.
    The main difference between these tasks and the PFL setup is the ability to select Pareto-optimal models at runtime.
    We direct the reader to \cite{NSFC20}, which provides a more comprehensive overview of recent work on these related problems.
    
    \subsubsection{Motivation and Objectives}
    Our work is, in spirit, a generalization of \cite{TW19}, which demonstrated a method for mapping out the primal DIB frontier for the special case of  binary classification (i.e., \(|Y| = 2\)).
    Although we deviate from their assumptions, assuming that \(X\) is discrete (rather than continuous in \cite{TW19}), and being limited to deterministic encodings (rather than stochastic ones in \cite{TW19}), and thus our results are not strictly comparable, our goal of mapping out the primal Pareto frontier is done in the same spirit.

    The most immediate motivation for mapping out the primal Pareto frontier is that its shape is useful for model selection: given multiple candidate solutions, each being near the frontier, we would often like to be able to privilege one over the others.
    For example, one typically favors the points that are the most robust to input noise, that is, those that are most separated from their neighbors, appearing as concave corners in the frontier.
    For the problem of geometric clustering with the Lagrangian DIB, the angle between piecewise linear portions of its frontier, known as the ``kink angle'', has been proposed as a criterion for model selection \cite{SS17a}.
    Using the primal DIB frontier, we can use distance from the frontier as a sharper criterion for model selection;
    this is particularly evident in the example discussed in Section~\ref{sec:group-dataset}, where the most natural solutions are clearly prominent in the primal frontier but have zero kink angle.
    The structure of this frontier also encodes useful information about the data.
    For clustering, corners in this frontier often indicate scales of interests: those at which the data best resolve into distinct clusters.
    Determining these scales is the goal of hierarchical clustering.

    Unlike the previously studied case of binary classification \cite{TW19}, no polynomial time algorithm is known for finding optimal clusterings for general \(|Y| > 2\) \cite{ZK16}.
    Finding an optimal solution to the DIB problem (i.e., one point near the frontier) is known to be equivalent to k-means in certain limits \cite{SB04,SS17a}, which is itself an NP-hard problem \cite{ACKS15}:
    mapping out the entirety of the frontier is no easier.
    More fundamentally, the number of possible encoders is known to grow super-exponentially with \(|X|\);
    therefore, it is not known whether we can even hope to store an entire DIB frontier in memory.
    Another issue is that of the generalization of the frontier in the presence of sampling noise.
    Estimation of mutual information and entropy for finite datasets is known to be a difficult problem with many common estimators either being biased, having high variance, or both \cite{NSB02, P03, KSG04, NBR04, POVA19}.
    This issue is of particular significance in our case as a noisy point on the objective plane can mask other, potentially more desirable, clusterings.

    It is these gaps in the optimization of DIB and DIB-like objectives that we seek to address.
    Firstly, existing work on optimization concerns itself only with finding a point on or near the frontier.
    These algorithms may be used to map out the Pareto frontier, but they need to be run multiple times with special care taken in sampling the constraint in order to attain the desired resolution of the frontier.
    Furthermore, we observe empirically that almost all of the DIB Pareto frontier is in fact convex.
    The majority of the existing algorithms applicable to DIB-like trade-offs optimize the Lagrangian DIB \cite{AFDM16, SS17b} and are therefore unable to capture the complete structure of the DIB frontier.
    Existing agglomerative methods \cite{ST99} are implicitly solving for the related but distinct DMC frontier, which has much less structure than the DIB frontier.
    Finally, existing methods have assumed access to the true distribution \(p_{XY}\) or otherwise used the maximum likelihood (ML) point estimators~\cite{SS17b, TW19}, which are known to be biased and have high variance for entropy and mutual information, which can have a significant effect on the makeup of the frontier.

    \subsection{Roadmap}
    The rest of this paper is organized as follows.
    In Section~\ref{sec:algorithms}, we tackle the issue of finding the Pareto frontier in practice by proposing a simple agglomerative algorithm capable of mapping out the Pareto frontier in one pass (Section~\ref{sec:pareto-mapper}) 
    and propose a modification that can be used to select robust clusterings with quantified uncertainties from the frontier when the sampling error is significant (Section~\ref{sec:robust-pareto-mapper}).
    We then present our analytic and experimental results in Section~\ref{sec:results}.
    In Section~\ref{sec:pareto-properties}, we provide evidence for the sparsity of the Pareto frontier, giving hope that it is possible to efficiently study it in practice.
    To demonstrate our algorithm and build intuition for the Pareto frontier structure, we apply it to three different DIB optimization problems in Section~\ref{sec:examples}: compressing the English alphabet, extracting informative color classes from natural images, and discovering group--theoretical structure.
    Finally, in Section~\ref{sec:discussion}, we discuss related results and directions for future work.
    

\section{Methods}
    \label{sec:algorithms}
    We design our algorithm around two main requirements: firstly, we would like to be able to optimize the primal objective of Equation~\eqref{eq:primalopt} directly, thereby allowing us to discover convex portions of the frontier; secondly, we would like a method that records the entire frontier in one pass rather than finding a single point with each run.
    While the task of finding the exact Pareto frontier is expected to be hard in general, Theorem~\ref{thm:scaling-by-area} applied to Example~\ref{ex:scaling-independent-rv}, gives us hope that the size of the Pareto frontier grows at most polynomially with input size \(|X|\).
    As is often the case when dealing with the statistics of extreme values, we expect that points near the frontier are rare and propose a pruned search technique with the hope that significant portions of the search space can be pruned away in practice.
    In the spirit of the probabilistic analysis provided above, we would like an algorithm that samples from a distribution that favors near-optimal encoders, thereby accelerating the convergence of our search.
    For this reason, we favor an agglomerative technique, with the intuition that there are good encoders that can be derived from merging the output classes of other good encoders.
    An agglomerative approach has the additional benefit of being able to record the entire frontier in one pass.
    For these reasons, we propose an agglomerative pruned search algorithm for mapping the Pareto frontier in Section~\ref{sec:pareto-mapper}.
    We also describe in Section~\ref{sec:robust-pareto-mapper} a modification of the algorithm that can be applied to situations where only a finite number of samples are available.

    \subsection{The Pareto Mapper}
    \label{sec:pareto-mapper}
    
    Our method, dubbed the \emph{Pareto Mapper} (Algorithm~\ref{alg:mapper}), is a stochastic pruned agglomerative search algorithm with a tunable parameter \(\epsilon\) that controls the search depth.
    The algorithm is initialized by enqueuing the identity encoder, \(f = \operatorname{id}\), into the search queue.
    At each subsequent step (illustrated in Figure~\ref{fig:pareto-pedagogy}), an encoder is dequeued.
    A set containing all of the Pareto-optimal encoders thus far encountered is maintained as the algorithm proceeds.
    All encoders that can be constructed by merging two of the given encoder's output classes (there are \(O(n^2)\) of these) are evaluated against the frontier of encoders seen so far; 
    we call encoders derived this way child encoders.
    If a child encoder is a distance \(d\) from the current frontier, we enqueue its children with probability \(e^{-d/\epsilon}\) and discard it otherwise, resulting in a search over an effective length-scale \(\epsilon\) from the frontier.
    The selection of \(\epsilon\) tunes the trade-off between accuracy and search time:
    \(\epsilon = 0\) corresponds to a greedy search that does not evaluate sub-optimal encodings, and \(\epsilon \rightarrow \infty\) corresponds to a brute-force search over all possible encodings.
    As the search progresses, the Pareto frontier is refined, and we are able to prune a larger majority of the proposed encoders.
    The output of our algorithm is a Pareto set of all found Pareto optimal clusterings of the given trade-off.
    \vspace{-6pt}
   \begin{figure}[H]
        \centering
        \includegraphics[width=12cm]{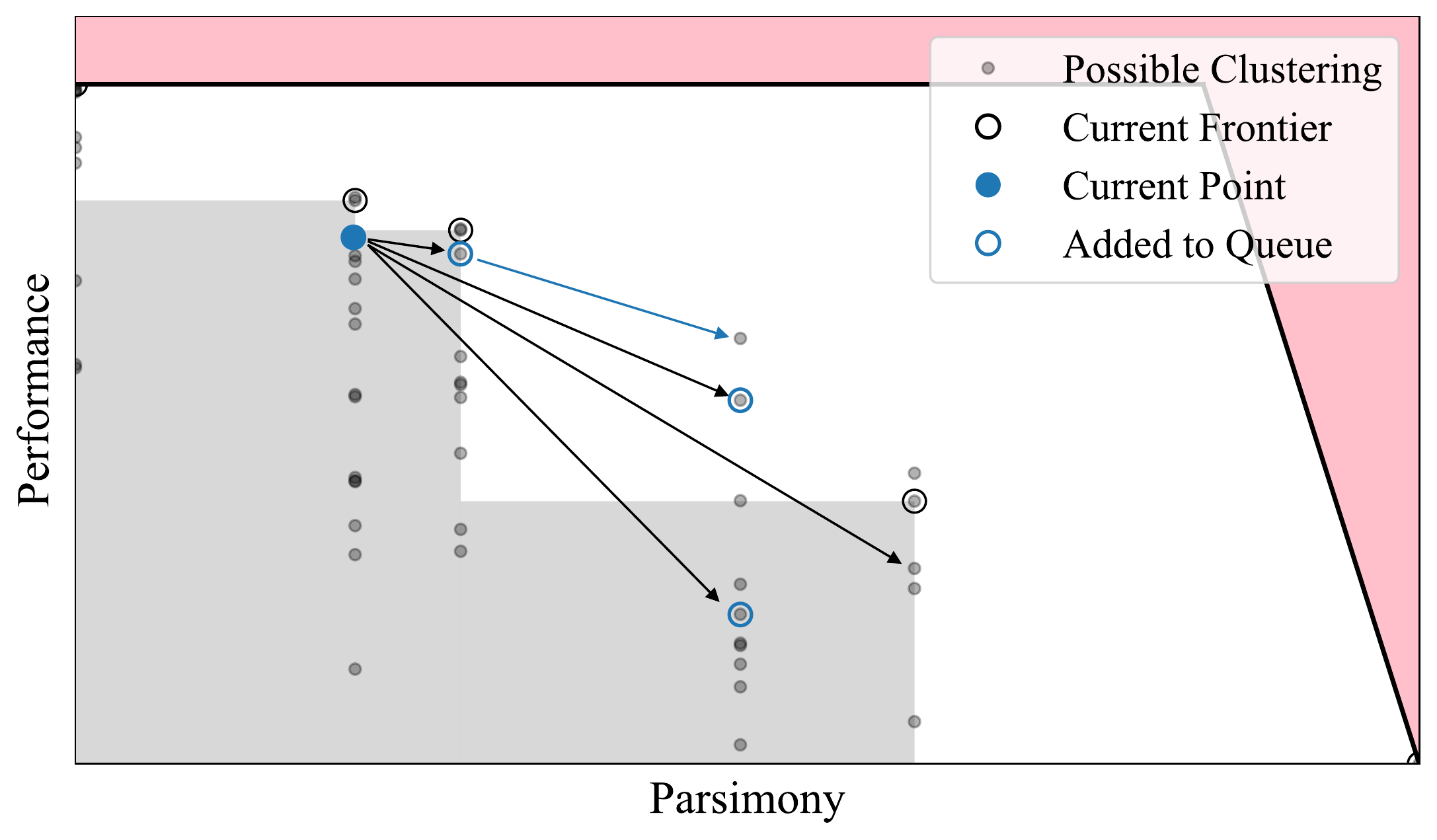}
        \caption{Illustration of one step in the main loop of the Pareto Mapper (Algorithm~\ref{alg:mapper}) mid-run.
        For pedagogical purposes, all possible encoders (filled gray circle) are plotted on the objective plane.
        The Pareto optimal points searched so far are marked with open black circles, and the region of the objective plane they dominate is shaded in gray.
        The black arrows show neighboring encoders and newly enqueued encoders are marked by open blue circles; 
        encodings that are optimal with respect to the current frontier are enqueued with certainty and sub-optimal encodings enqueued with probability \(e^{-d/\epsilon}\), where \(d\) is the distance from the frontier.
        Note that some Pareto optimal points are only accessible through sub-optimal encoders (blue arrow).
        }
        \label{fig:pareto-pedagogy}
    \end{figure}
    The Pareto frontier at any given moment is stored in a data structure, called a \emph{Pareto set}, which is a list of point structures.
    A point structure, \(p\), contains fields for both objectives \(p\mathrm{.x}\), \(p\mathrm{.y}\), and optional fields for storing the uncertainties \(p\mathrm{.dx}\), \(p\mathrm{.dy}\) and clustering function \(p\mathrm{.f}\).
    The Pareto set is maintained so that the Pareto-optimality of a point can be checked against a Pareto set of size \(m\) in \(\Theta(\log m)\) operations.
    Insertion into the data structure requires in the worst case \(\Theta(m)\) operations, which is optimal, as a new point could dominate \(\Theta(m)\) existing points necessitating their removal.
    We define the distance from the frontier as the minimum Euclidean distance that a point would need to be displaced before it is Pareto-optimal, which also requires in the worst case \(\Theta(m)\) operations.
    A list of pairs \((H(Z), I(Z; Y))\), sorted by its first index, provides a simple implementation of the Pareto set.
    The pseudocode for important auxiliary functions such as \(\textsc{pareto\_add}\) and \(\textsc{pareto\_distance}\) is provided in Appendix~\ref{app:auxiliary-functions}.

    Although we have provided evidence for polynomial scaling of size of the Pareto set, it is not obvious if the polynomial scaling of the Pareto set translates to the polynomial scaling of our algorithm, which depends primarily on how quickly the search space can be pruned away by evaluation against the Pareto frontier.
    To demonstrate the polynomial scaling of our algorithm with \(n\), we evaluate the performance of the Pareto Mapper on randomly generated \(p_{XY}\).
    Since \(\epsilon \rightarrow \infty\) corresponds to a brute-force search, and therefore has no hope of having polynomial runtime, we focus on the \(\epsilon \rightarrow 0\) case;
    we show later, in Section~\ref{sec:alphabet-dataset}, that \(\epsilon \rightarrow 0\) is often sufficient to achieve good results.
    For Figure~\ref{fig:scaling-random-data}a, we randomly sample \(p_{XY}\) uniformly over the simplex of dimension \(|X| |Y| - 1\) varying \(|X|\) with fixed \(|Y| = 30\).
    We find that the scaling is indeed polynomial.
    Comparing with the scaling of the size of the Pareto set shown in Figure~\ref{fig4}b, we see that approximately \(O(n)\) points are searched for each point on the Pareto frontier.
    While the computation time, empirically estimated to be \(\Theta(n^{5.0})\), is limiting in practice, we note that it is indeed polynomial, which is sufficient for our purposes.

    \begin{figure}[H]
        \centering 
        {\captionsetup{position=bottom,justification=centering}  \begin{subfigure}[b]{0.48\textwidth}
                \includegraphics[height=5.5cm]{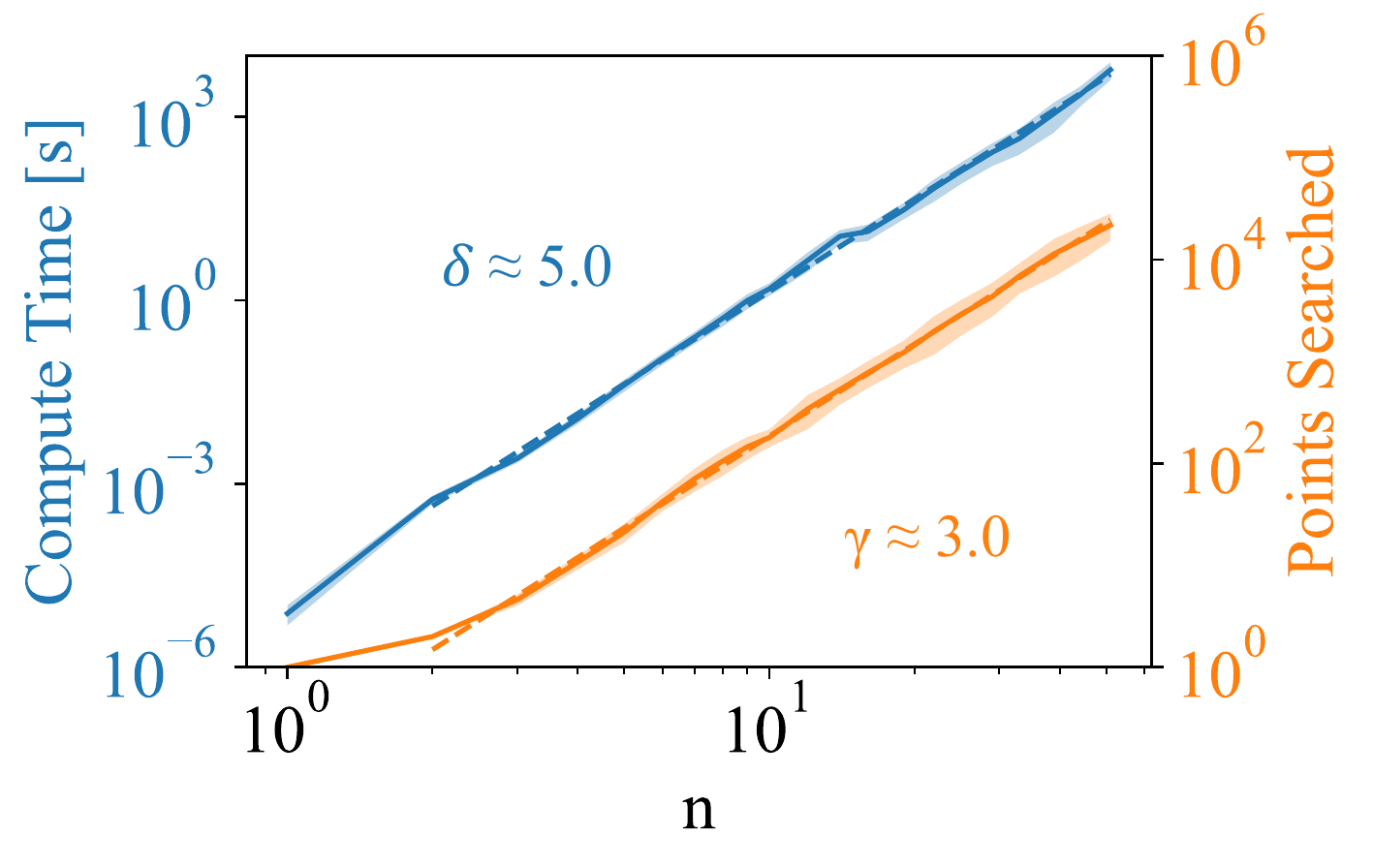}
                \caption{\label{fig:scaling-random-data-n}}
            \end{subfigure}
            ~ 
            \begin{subfigure}[b]{0.48\textwidth}
                \includegraphics[height=5.5cm]{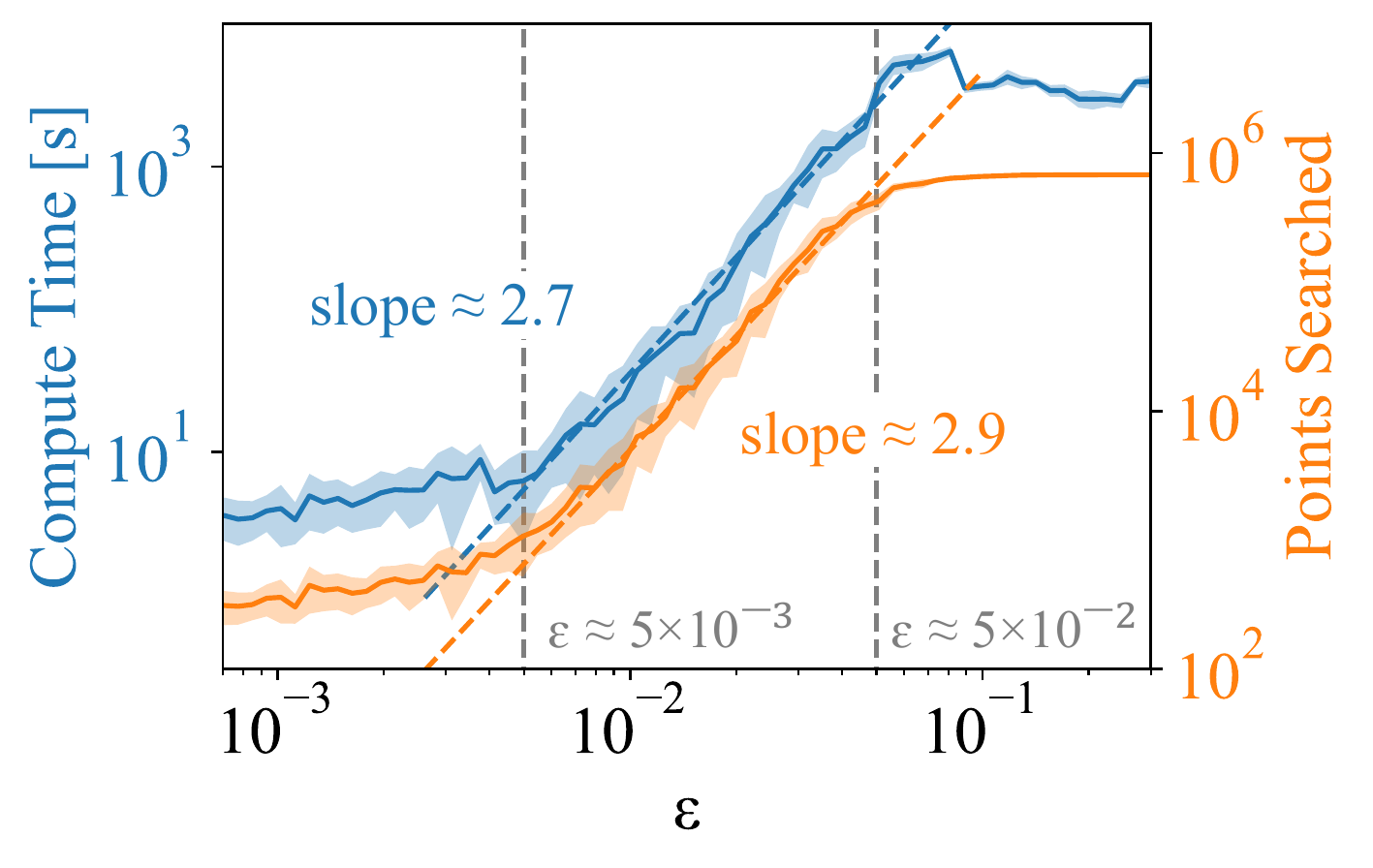}
                \caption{\label{fig:scaling-random-data-epsilon}}
            \end{subfigure}}
           \caption{Scaling of computation time (\textbf{left} scale) and points searched (\textbf{right} scale) as a function of (\textbf{a})~input size \(n\) at \(\epsilon = 0\), where we find compute time scales as \(O(n^\delta)\) and the size of the Pareto set scales as \(O(n^\gamma)\); (\textbf{b}) search parameter \(\epsilon\) for randomly generated \(p_{XY}\) of fixed size.}
           \label{fig:scaling-random-data}
    \end{figure}
    
    \begin{figure}[H]
        \centering
           {\captionsetup{position=bottom,justification=centering} \begin{subfigure}[b]{0.48\textwidth}
                \includegraphics[height=5.5cm]{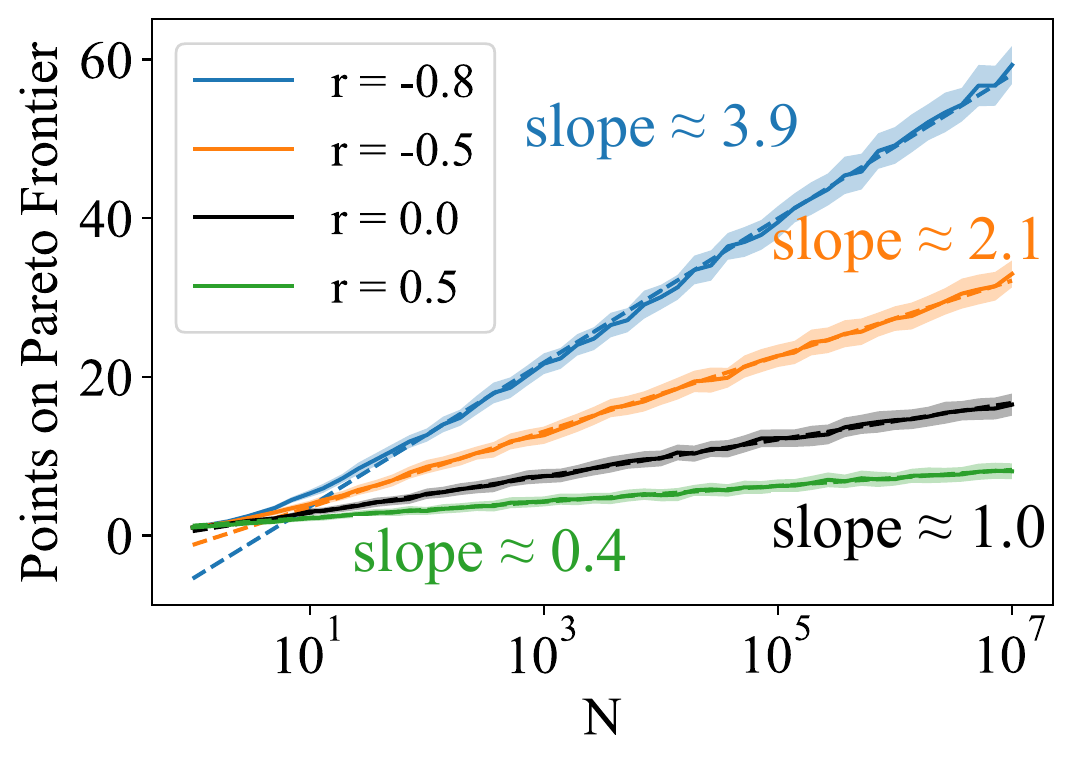}
                \caption{}
                \label{fig:Nscaling}
            \end{subfigure}%
            ~ 
            \begin{subfigure}[b]{0.48\textwidth}
                \includegraphics[height=6.25cm]{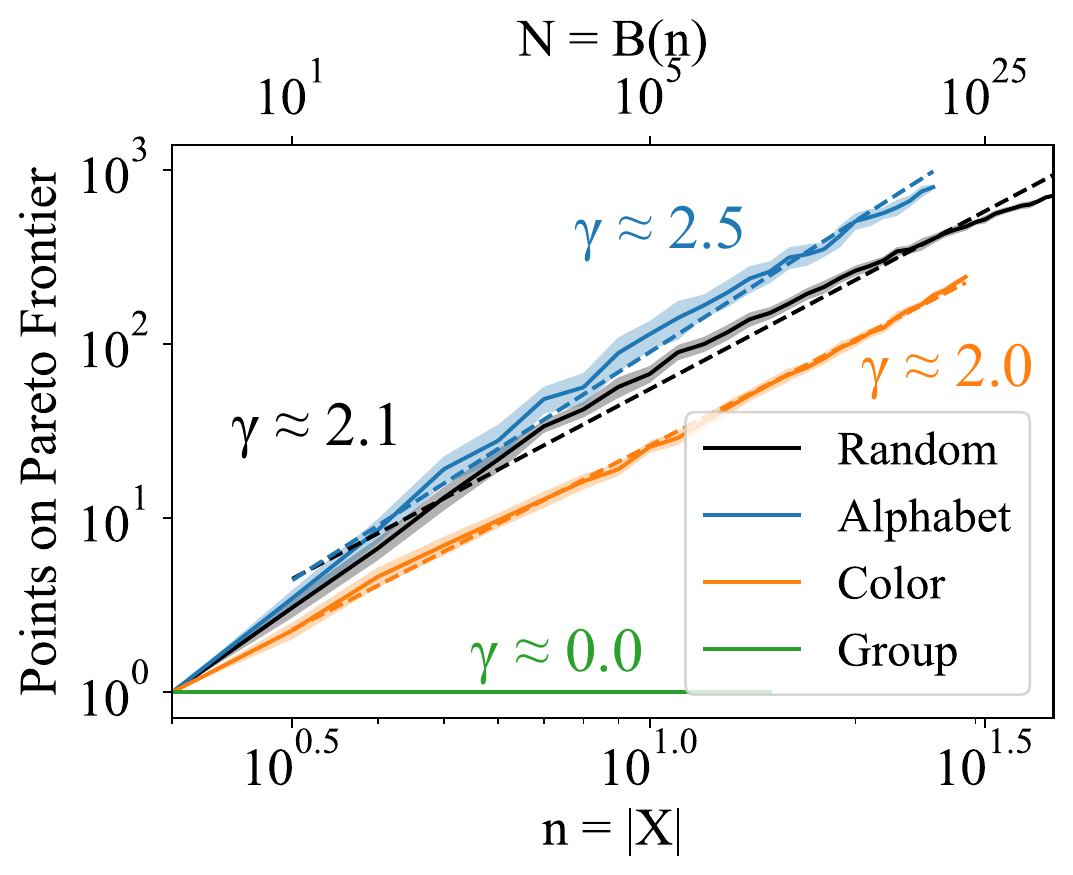}
                \caption{}
                \label{fig:nscaling-dibpareto}
            \end{subfigure}}
            \caption{Scaling of number of points on the Pareto frontier (\textbf{a}) as a function of \(N = |S|\) for bivariate Gaussian distributed \((U, V)\) with specified correlation \(r \equiv \sigma_{UV} / \sigma_U \sigma_V\), and (\textbf{b}) for the DIB problem with input size \(n\) where we find \(|\operatorname{Pareto}(S)| = O(n^\gamma)\).}
            
            \label{fig4}
    \end{figure}
    We also evaluate the scaling of our algorithm with \(\epsilon\).
    Again, we randomly sample \(p_{XY}\) uniformly over the simplex of dimension \(|X| |Y| - 1\), this time fixing \(|X| = 11\) and \(|Y| = 30\), with results plotted in Figure~\ref{fig:scaling-random-data}b.
    We find that the relevant scale for distances is between \(\epsilon_- \approx 5 \times 10^{-3}\) and \(\epsilon_+ \approx 5 \times 10^{-2}\); while the specifics of the characteristic range for \(\epsilon\) depends on the dataset, we empirically find that while \(\epsilon_+\) remains constant, \(\epsilon_-\) decreases as \(n\) increases.
    This is consistent with the fact that as \(n\) increases, the DIB plane becomes denser, and the average separation between points decreases.
    This would suggest that there is an \(n\) above which the runtime of the Pareto Mapper exhibits exponential scaling for any \(\epsilon > 0\).
    In the absence of noise, one can run the Pareto Mapper at a number of different values of \(\epsilon\) evaluating precision and recall with respect to the combined frontier to evaluate convergence (see Figure \ref{fig:performance-sweeps}b). 
    We discuss how to set \(\epsilon\) in the presence of noise due to sampling in Section~\ref{sec:alphabet-dataset}.

    \subsection{Robust Pareto Mapper}
    \label{sec:robust-pareto-mapper}
    So far, we have assumed access to the true joint distribution \(p_{XY}\).
    Normally, in practice, we are only provided samples from this distribution and must estimate both objective functions, the mutual information \(I(Z; Y)\) and entropy \(H(Z)\), from the empirical distribution \(\hat{p}_{XY}\).
    Despite the uncertainty in these estimates, we would like to find clusterings that are Pareto optimal with respect to the true distribution.
    Here, we propose a number of modifications to Algorithm~\ref{alg:mapper} that allow us to quantify our uncertainty about the frontier and thereby produce a frontier that is robust to sampling noise.
    The modified algorithm, dubbed the \emph{Robust Pareto Mapper} (Algorithm~\ref{alg:robust-mapper}), is described below.

    Given samples from the joint distribution, we construct the empirical joint distribution and run the Pareto Mapper (Algorithm~\ref{alg:mapper}) replacing the entropy and mutual information functions with their respective estimators.
    We use the entropy estimator due to Nemenman, Shafee, and Bialek (NSB) \cite{NSB02}, as it is a Bayesian estimator that provides not only a point estimate but also provides some bearing on its uncertainty, although another suitable estimator can be substituted.
    We find that our method works well even with point estimators, in which case resampling techniques (e.g., bootstrapping) are used to obtain the uncertainty in the estimate.
    After running the Pareto Mapper, points that are not significantly different from other points in the Pareto set are removed.
    This filtering operation considers points in order of ascending uncertainty, as measured by the product of its standard deviations in \(H(Z)\), and \(I(Z; Y)\).
    Subsequent points are added as long as they do not overlap with the confidence interval in either \(H\) or \(I\) with a previously added point, and they are removed otherwise.
    There is some discretion in choosing the confidence interval, which we have chosen empirically to keep the discovered frontier robust between independent runs.
    This filtering step is demonstrated in Section~\ref{sec:alphabet-dataset}.

    \begin{algorithm}[H]
        \caption{Robust Pareto Mapper: dealing with finite data}
        \label{alg:robust-mapper}        
          \textit{Input}: Empirical joint distribution \(\hat{p}_{XY}\), search parameter \(\varepsilon\), and sample size \(S\)          
          
        \textit{Output}: Approximate Pareto frontier \(P\) with uncertainties
        \begin{algorithmic}[1]
            \Procedure{Robust\_Pareto\_Mapper}{$\hat{p}_{XY}, \varepsilon$}
                \State \textbf{Pareto set} \(P \leftarrow \textsc{pareto\_mapper}(\hat{p}_{XY}, \epsilon)\) \Comment{Run \textsc{pareto\_mapper} with suitable estimators}
                \State \textbf{Pareto set} \(P' = \emptyset\) \Comment{Initialize set of robust encoders}
                
                \For{\(p \in P\)} \Comment{This step can be skipped if an interval estimator is used above}
                    \State \((p\mathrm{.dx}, p\mathrm{.dy}) \leftarrow \textsc{resample}(p, \hat{p}_{XY})\) \Comment{Store uncertainty of points on frontier}
                \EndFor

                \For{\(p \in P\) in ascending order of uncertainty}
                    \If{\(p\) is significantly different than all \(q \in P'\)} 
                        \State \(P' \leftarrow \textsc{pareto\_add}(p, P')\) \Comment{Filter with preference for points with low variance}
                    \EndIf
                \EndFor

                \Return{\(P'\)}
            \EndProcedure
        \end{algorithmic}
    \end{algorithm}


\section{Results}
    \label{sec:results}
    \subsection{General Properties of Pareto Frontiers}
    \label{sec:pareto-properties}
    Before introducing the specifics of the DIB problem, we would like to understand a few general properties of the Pareto frontier.
    The most immediate challenge we face is the size of our search space.
    For an input of size \(|X|\), the number of points on the DMC frontier is bounded by \(|X|\), but there is no such limit on the DIB frontier.
    Given the combinatorial growth of the number of possible clusterings with \(|X|\), it is not immediately clear that it is possible to list all of the points on the frontier, let alone find them.
    If we are to have any chance at discovering and specifying the DIB frontier, it must be that DIB-optimal points are sparse within the space of possible clusterings, where sparse is taken to mean that the size of the frontier scales at most polynomially with the number of items to be compressed.

    In this section, we provide sufficient conditions for the sparsity of the Pareto set in general and present a number of illustrative examples.
    We then apply these scaling relationships to the DIB search space and provide numerical evidence that the number of points grows only polynomially with \(n \equiv |X|\) for most two-objective trade-off tasks.

    \subsubsection{Argument for the Sparsity of the Pareto Frontier}
    First, we will formally define a few useful terms.
    Let \(S = \{\vec{s_i}\}_{i=1}^N\) be a sample of \(N\) independent and identically distributed (i.i.d.) bivariate random variables representing points \(\vec{s_i} = (U_i, V_i)\) in the Pareto plane.

    \begin{definition}
        A point \((U, V) \in S\) is \textit{maximal} with respect to \(S\), or equivalently called \textit{Pareto}-optimal, if \(\forall (U_i, V_i) \in S\), \(\exists V_i > V \implies U > U_i\).
        In other words, a point is maximal with respect to \(S\) if there are no points in \(S\) both above and to its left.
    \end{definition}

    \begin{definition}
        For a set of points \(S \subset \R^2\), the \textit{Pareto set} \(\operatorname{Pareto}(S) \subseteq S\) is the largest subset of \(S\) such that all \((U, V) \in \operatorname{Pareto}(S)\) are maximal with respect to \(S\).
    \end{definition}

    Now, we can state the main theorem of this section, which we prove in Appendix~\ref{app:pareto-size}.
    \begin{restatable}[]{theorem}{scalingbyarea}
    \label{thm:scaling-by-area}
        Let  \(S = \{(U_i, V_i)\}_{i=1}^N\) be a set of bivariate random variables drawn i.i.d. from a distribution with Lipschitz continuous CDF \(F(u, v)\), and invertible marginal CDFs \(F_U, F_V\).
        Define the region 
        \begin{equation}
            R_N \equiv \left\{(u, v) \in [0, 1] \times [0, 1] : u + v - C(u, v) \ge e^{-\frac{1}{N}}\right\}
        \end{equation} 
        where \(C(u, v)\) denotes the copula of \((U_i, V_i)\), which is the function that satisfies \(F(u, v) = C(F_U(u), F_V(v))\).

        Then, if the Lebesgue measure of this region \(\lambda(R_N) = \Theta\left(\frac{\ell(N)}{N}\right)\) as \(N \rightarrow \infty\), we have \(\E\left[|\operatorname{Pareto}(S)|\right] = \Theta(\ell(N))\).
    \end{restatable}

    \begin{example}
    Let us consider the case of independent random variables with copula \(C(u, v) = uv\).
    Note that in this case, the level curves \(u + v - C(u, v) = e^{-\frac{1}{N}}\) are given by \(v = \frac{e^{-\frac{1}{N}} - u}{1 - u}\).
    We can then integrate to find the area of the region \(R_N\)
    \begin{equation}
        \lambda(R_N) = 1 - \int_{0}^{e^{-\frac{1}{N}}} \frac{e^{-\frac{1}{N}} - u}{1 - u} du = e^{-1/n} \left(1 - e^{1/n}\right) \left(\log \left(1-e^{-1/n}\right)-1\right)
    \end{equation}
    Expanding for large \(N\), we find that \(\lambda(R_N) = \frac{\log N}{N} + O(N^{-1})\).
    We see that this satisfies the conditions for Theorem~\ref{thm:scaling-by-area} with \(\ell(N) = \log N\), giving \(\E_{S}\left[|\operatorname{Pareto}(S)|\right] = \Theta(\log N)\).
    \label{ex:scaling-independent-rv}
    \end{example}

    Additional examples can be found in Appendix~\ref{app:pareto-size}.
    Numerically, we see that for independent random variables \(U\) and \(V\), the predicted scaling holds even down to relatively small \(N\);
    furthermore, the linear relationship also holds for correlated Gaussian random variables \(U, V\) (Figure~\ref{fig4}a).
    The logarithmic sparsity of the Pareto frontier allows us to remain hopeful that it is possible, at least in principle, to fully map out the DIB frontier for deterministic encodings of discrete domains despite the super-exponential number of possible encoders.

    \subsubsection{Dependence on Number of Items \(|X|\)}
    The analysis above gives us hope that Pareto-optimal points are generally polylogarithmically sparse in \(N \equiv |S|\), i.e., 
    \(|\operatorname{Pareto}(S)| = O((\log N)^\gamma)\).
    Of course, the scenario with which we are concerned is one where the random variables \(U = -H(Z)\) and \(V = I(Z; Y)\);
    the randomness of \(U\) and \(V\) in this case comes from the choice of encoder \(f: X \rightarrow Z\) which, for convenience, we assume is drawn i.i.d. from some distribution over the space of possible encodings.
    Note that conditioned on the distribution of \(X\) and \(Y\), the points \((H(f(X)), I(f(X); Y))\) are indeed independent, although the agglomerative method we use to sample from the search space introduces dependence;
    however, in our case, this dependence likely helps the convergence of the algorithm.

    In the DIB problem, and clustering problems more generally, we define \(n \equiv |X|\) to be the size of the input.
    The search space is over all possible encoders \(f: X \rightarrow Z\), which has size \(N = B(n)\) where \(B(n)\) are the Bell numbers.
    Asymptotically, the Bell numbers grow super-exponentially: \(\ln{B(n)} \sim n\ln{n} - n\ln{\ln{n}}\), making an exhaustive search for the frontier intractable.
    We provide numerical evidence in Figure~\ref{fig4}b that the sparsity of the Pareto set holds in this case, with its size scaling as \(O(\operatorname{poly}(n))\), or equivalently, \(O(\operatorname{polylog}(N))\).
    While in the worst case, all \(B_n\) clusterings can be DIB-optimal (the case where \((p_{XY})_{ij} = \operatorname{diag}(\vec{r})\) for \(r_i\) drawn randomly from the \((n - 1)\)-dimensional simplex results in clusterings with strict negative monotonic dependence on the DIB plane, and therefore all points are DIB-optimal, see Appendix~\ref{app:pareto-size}), our experiments show that in practice, the size of the Pareto frontier (and compute time) grows polynomially with the number of input classes \(n\) (Figure \ref{fig:scaling-random-data}), with the degree of the polynomial depending on the details of the joint distribution \(p(x, y)\).
    This result opens up the possibility of developing a tractable heuristic algorithm that maps out the Pareto frontier, which we will explore in the remainder of this paper.

    \subsection{At the Pareto Frontier: Three Vignettes}
    \label{sec:examples}
    The purpose of this section is to demonstrate our algorithm and illustrate how the primal DIB frontier can be used for model selection and to provide additional insights about the data.
    To this end, we apply our algorithm to three different DIB optimization tasks: predicting each character from its predecessor in the English text, predicting an object from its color, and predicting the output of a group multiplication given its input.
    In all cases, the goal is to find a representation that retains as much predictive power as possible given a constraint on its entropy.
    We will describe the creation of each dataset and motivate its study.
    For all three tasks, we discuss the frontier discovered by our algorithm and highlight informative points on it, many of which would be missed by other methods either because they are not DMC-optimal or because they lie within convex regions of the frontier.
    For the task of predicting the subsequent English character, we will also compare our algorithm to existing methods including the Blahut--Arimoto algorithm, and a number of geometric clustering algorithms;
    we will also use this example to demonstrate the Robust Pareto Mapper (Algorithm~\ref{alg:robust-mapper}).

    \subsubsection{Compressing the English Alphabet}
    \label{sec:alphabet-dataset}

    First, we consider the problem of compressing the English alphabet with the goal of preserving the information that a character provides about the subsequent character.
    In this case, we collect \(2\) gram statistics from a large body of English text.
    Treating each character as a random variable, our goal is to map each English character \(X\) into a compressed version \(Z\) while retaining as much information as possible about the following character \(Y\).

    Our input dataset is a $27\times 27$ matrix of bigram frequencies for the letters a--z and the space character, which we denote ``\_'' in the figures below. 
    We computed this matrix from the 100 Mb {\it enwiki8} ({\url{http://prize.hutter1.net/}}) Wikipedia corpus after removing all symbols except letters and the space character, removing diacritics, and making all letters lower-case.

    The Pareto frontier is plotted in Figure~\ref{fig:alphabet-frontier} and the points corresponding to some interesting clusterings on the frontier are highlighted.
    We find that from \(2\) gram frequencies alone, the DIB-optimal encodings naturally discover relevant linguistic features of the characters, such as their grouping into vowels and consonants.

    \begin{figure}[H]
        \centering 
            \includegraphics[width=15.5cm]{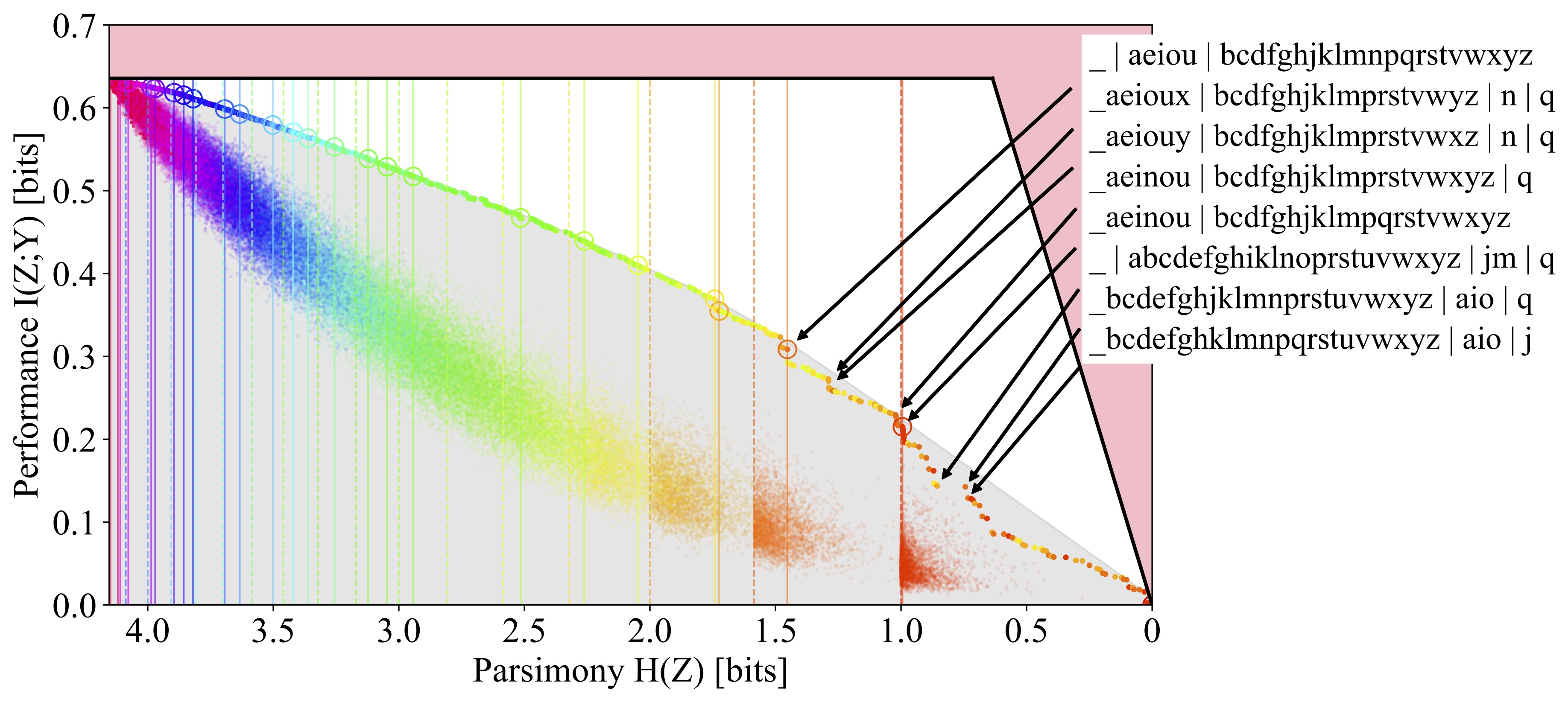}
            \caption{The primal DIB frontier of the English alphabet is compressed to retain information about the subsequent character.
            Points in the shaded gray region, indicating the convex hull, are missed by optimization of the DIB Lagrangian.
            Encodings corresponding to interesting features of the frontier are annotated, and DMC-optimal points are circled.
            Dotted vertical lines mark the location of balanced clusters (i.e., \(H(Z) = \log_2 k\) for \(k \in \{2, 3, \ldots\}\)), and solid vertical lines correspond to the entropy of the DMC-optimal encodings.
            A sample of encodings drawn uniformly at random for each \(|Z|\) is evaluated on the plane, illustrating the large distance from the frontier for typical points in the search space.
            The color indicates \(|Z|\).
            }
            \label{fig:alphabet-frontier}
    \end{figure}

    The DMC-optimal encoding corresponding to a cluster size of \(k = 2\) is nearly balanced (i.e., \(H(Z) = \log_2 2\)) and separates the space character and the vowels from most of the consonants.
    However, in contrast to the binary classification case of \(|Y| = 2\) studied in \cite{TW19}, the DMC-optimal encodings are far from balanced (i.e., \(H(Z) \approx \log_2(k)\)) for larger \(k\).
    We note that on the DIB frontier, the most prominent corners are often not the DMC-optimal points, which are circled.
    By looking at the corners and the DMC-optimal points near the corners, which are annotated on the figure, we discover the reason for this:
    distinguishing anomalous letters such as `q' has an outsized effect on the overall information relative to its entropy cost.
    These features are missed when looking only at DMC-optimal points, because although the \(2\) gram statistics of `q' are quite distinct (it is almost always followed by a `u'), it does not occur frequently enough to warrant its own cluster when our constraint is cluster count rather than entropy.
    In other words, `q` is quite special and noteworthy, and our Pareto plot reveals this while 
    the traditional DMC or DIB plots do not. 

    The frontier is seen to reveal features at multiple scales, the most prominent corner corresponding to the encoding that separates the space character, `\_', from the rest of the alphabet, and the separation of the vowels from (most of the) consonants.
    The separation of `q` 
    often results in a corner of a smaller scale because it is so infrequent.
    These corners indicate the natural scales for hierarchical clustering.
    We also note that a large majority of the points, including those highlighted above, are below the convex hull (denoted by the solid black line)
    and are therefore missed by algorithms that optimize the DIB Lagrangian.

    A random sample of clusterings colored by \(|Z|\) is also plotted on the DIB plane in Figure~\ref{fig:alphabet-frontier};
    a sample for each value of \(|Z| = \{1, \ldots, |X|\}\) is selected uniformly at random.
    We see that there is a large separation between the typical clustering, sampled uniformly at random, and the Pareto frontier, indicating that a pruned search based on the distance from the frontier, such as the Pareto Mapper of Algorithm~\ref{alg:mapper}, is likely able to successfully prune away much of the search space.
    A better theoretical understanding of the density could provide further insights on how the runtime scales with \(\epsilon\).

    We now compare the results of the Pareto Mapper (Algorithm~\ref{alg:mapper}) with other clustering methods.
    We first use the Pareto Mapper with \(\epsilon = 0\) to derive a new dataset from the original alphabet dataset (which has \(|X| = 27\)) by taking the \(|Z| = 10\) clustering with the highest mutual information.
    We are able to obtain a ground truth for this new dataset with \(|X| = 10\) using a brute force search, against which we compare the other methods.
    These methods are compared on the DIB plane in Figure~\ref{fig:pruning-compare} and in tabular form in Table~\ref{tab:pruning-compare}.
    Notably, all the encodings found by the Blahut--Arimoto algorithm used in \cite{TPB00, SS17b} are DIB-optimal, but as it optimizes the DIB Lagrangian, it is unable to discover the convex portions of the frontier.
    We also compare our algorithm to geometric clustering methods where we assign clusters pairwise distances according to the Jensen--Shannon distances between the conditional distributions \(p(Y | X = x_i)\).
    These methods perform poorly when compared on the DIB plane for a number of reasons:
    firstly, some information is lost in translation to a geometric clustering problem, since only pairwise distances are retained;
    secondly, the clustering algorithms are focused on minimizing the number of clusters and are therefore unable to find more than \(n\) points.
    Additionally, these geometric clustering algorithms, while similar in spirit, are not directly optimizing the DIB objective.

    \begin{figure}[H]
        \centering
        \includegraphics[width=13cm]{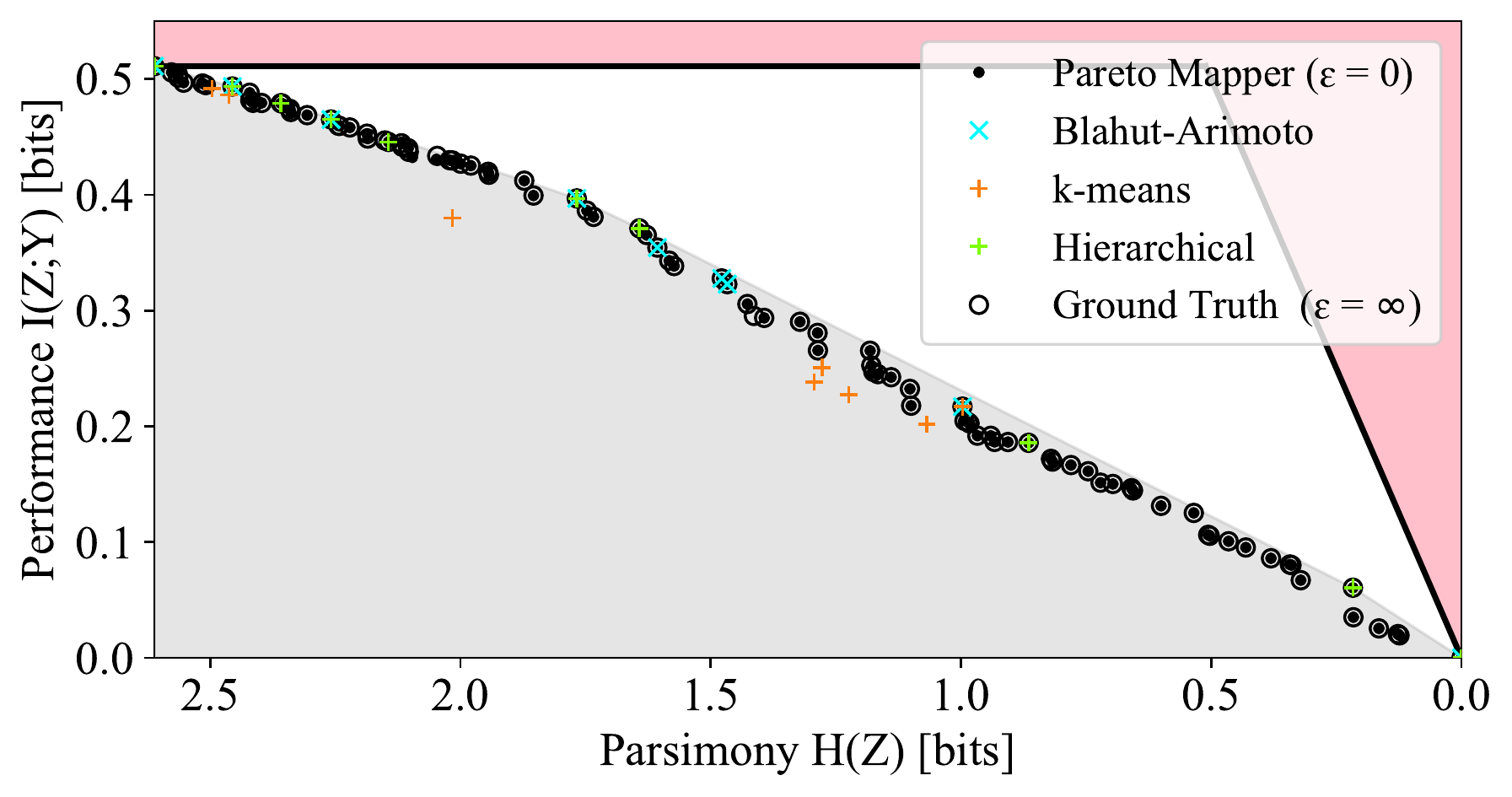}
        \caption{Comparison of the Pareto Mapper and other classification algorithms with ground truth for \(|X| = 10\). 
        The true Pareto frontier is calculated with a brute force search over all \(B(10) = 115,975\) clusterings \(f\).}
        \label{fig:pruning-compare}
    \end{figure}

    \begin{table}[H]
        \caption{Comparison of the performance of Algorithm~\ref{alg:mapper} with other clustering algorithms. Here, a true positive (TP) is a point that is correctly identified as being Pareto optimal by a given method; false positives (FP) and false negatives (FN) are defined analogously.}
        \setlength{\tabcolsep}{3.5mm}
        \label{tab:pruning-compare}
        \begin{tabular}{lrrrrrr}
            \hline
            \textbf{Method} & \textbf{Points} & \textbf{TP} & \textbf{FP} & \textbf{FN} & \textbf{Precision} & \textbf{Recall} \\
            \hline
            Ground truth (\(\epsilon \rightarrow \infty\)) & 94 & 94 & 0 & 0 & 1.00 & 1.00 \\
            Pareto Mapper (\(\epsilon = 10^{-2}\)) & 94 & 94 & 0 & 0 & {\bf 1.00} & {\bf 1.00} \\    
            \hline
            Pareto Mapper (\(\epsilon = 0\)) & 91 & 88 & 3 & 6& 0.97 & 0.94\\
            \hline
            Hierarchical (average) & 10 & 7 & 3 & 87 & 0.70 & 0.07 \\
            Hierarchical (single) & 10 & 10 & {\bf 0} & 84 & {\bf 1.00} & 0.11 \\
            Hierarchical (Ward) & 10 & 7 & 3 & 87 & 0.70 & 0.07 \\
            \hline
            k-means (JSD) & 10 & 3 & 7 & 91 & 0.30 & 0.02 \\
            k-means (wJSD) & 10 & 2 & 8 & 92 & 0.20 & 0.10 \\
            \hline
            Blahut Arimoto & 9 & 9 & {\bf 0} & 85 & {\bf 1.00} & 0.10 \\
            \hline
        \end{tabular}
    \end{table}

    To demonstrate the Robust Pareto Mapper (Algorithm~\ref{alg:robust-mapper}), we create a finite sample \(\hat{n}_{XY} = s \hat{p}_{XY}\) from a multinomial distribution with parameter \(p_{XY}\) and \(s\) trials.
    To quantify the sample size in natural terms, we define the sampling ratio \(r \equiv s / 2^{H(X, Y)}\).
    The results of the Robust Pareto Mapper on the alphabet dataset for several sampling ratios are shown in Figure~\ref{fig:robust-maps}.
    We note that even for relatively low sampling ratios, the algorithm is able to extract interesting information;
    it is able to quickly separate statistically distinct letters such as `q' and identify groups of characters such as vowels.
    As the sampling ratio increases, the Robust Pareto Mapper identifies a larger number of statistically significant clusterings (marked in red) from the rest of the discovered frontier (marked in gray).
    It is also notable that uncertainties in the entropy are typically lowest for encodings that split \(X\) into roughly equally probably classes;
    that these clusters are preferred is most readily seen in the highlighted clustering with \(H(Z) \approx 1\) in Figure~\ref{fig:robust-maps}b.
    We can see from these plots that, especially for low sampling ratios, the estimated frontier often lies above that of the true \(p_{XY}\) (solid black line).
    This is expected, as estimators for mutual information are often biased high.
    Despite this, the true frontier is found to lie within our estimates when the variance of the estimators is taken into account even for modest sampling ratios, as seen in the plot for \(r = 4\).

    \begin{figure}[H]
        \centering
          {\captionsetup{position=bottom,justification=centering}  \begin{subfigure}[b]{0.32\textwidth}
                \includegraphics[height=5.5cm]{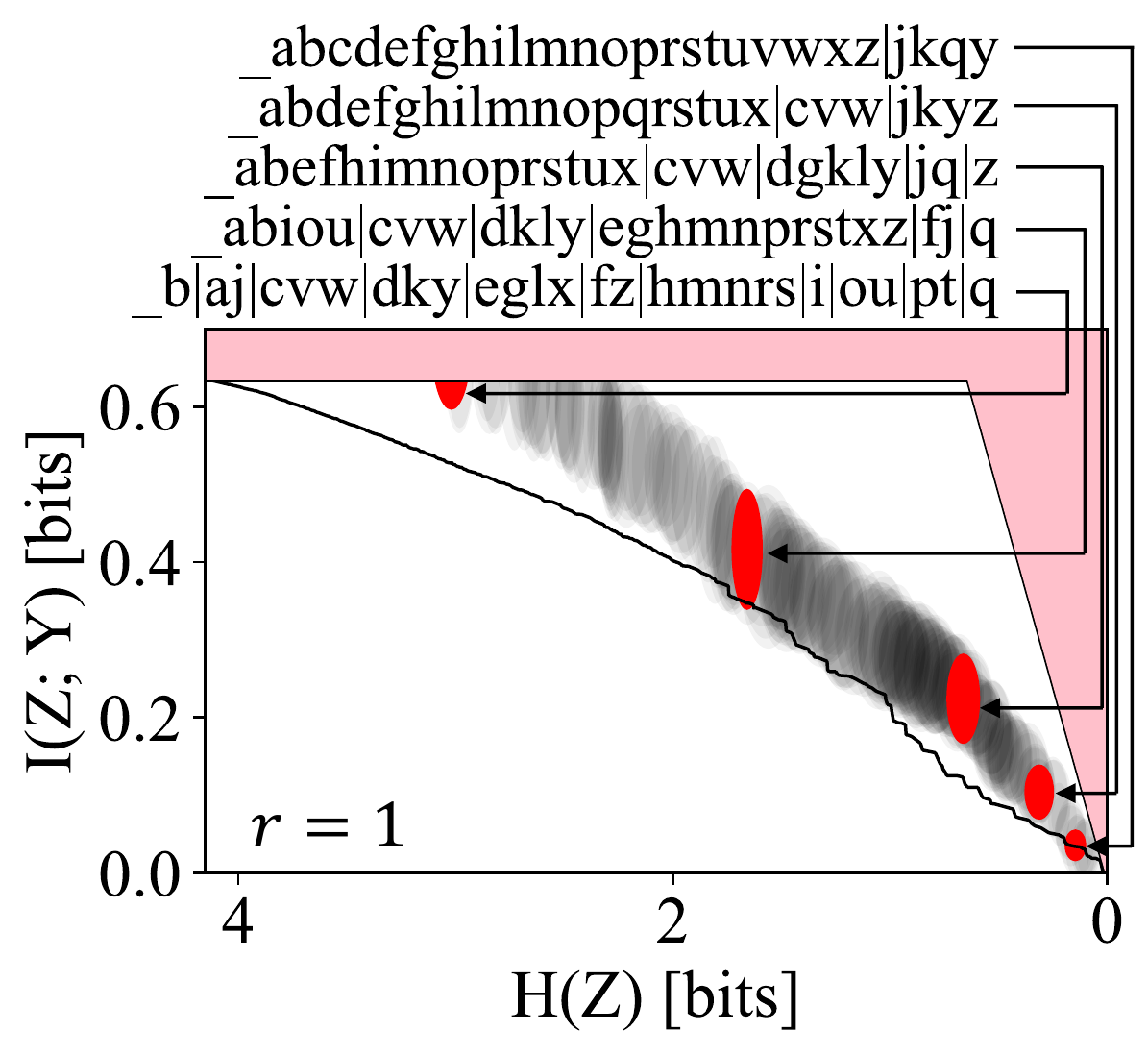}
                \caption{\label{fig:robust-maps-r01}}
            \end{subfigure}
            ~ 
            \begin{subfigure}[b]{0.32\textwidth}
                \includegraphics[height=5.5cm]{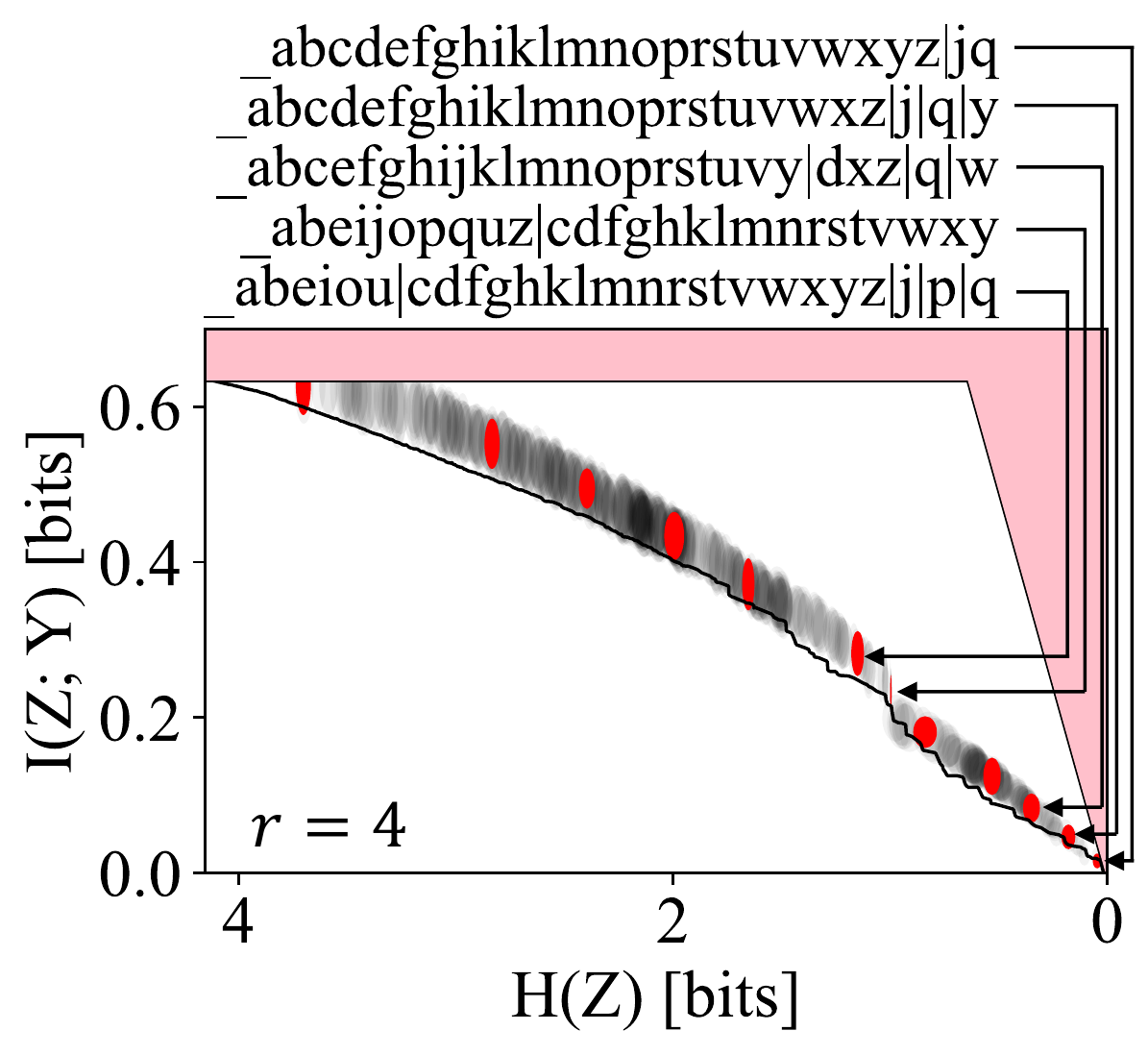}
                \caption{\label{fig:robust-maps-r04}}
            \end{subfigure}
            \begin{subfigure}[b]{0.32\textwidth}
                \includegraphics[height=5.5cm]{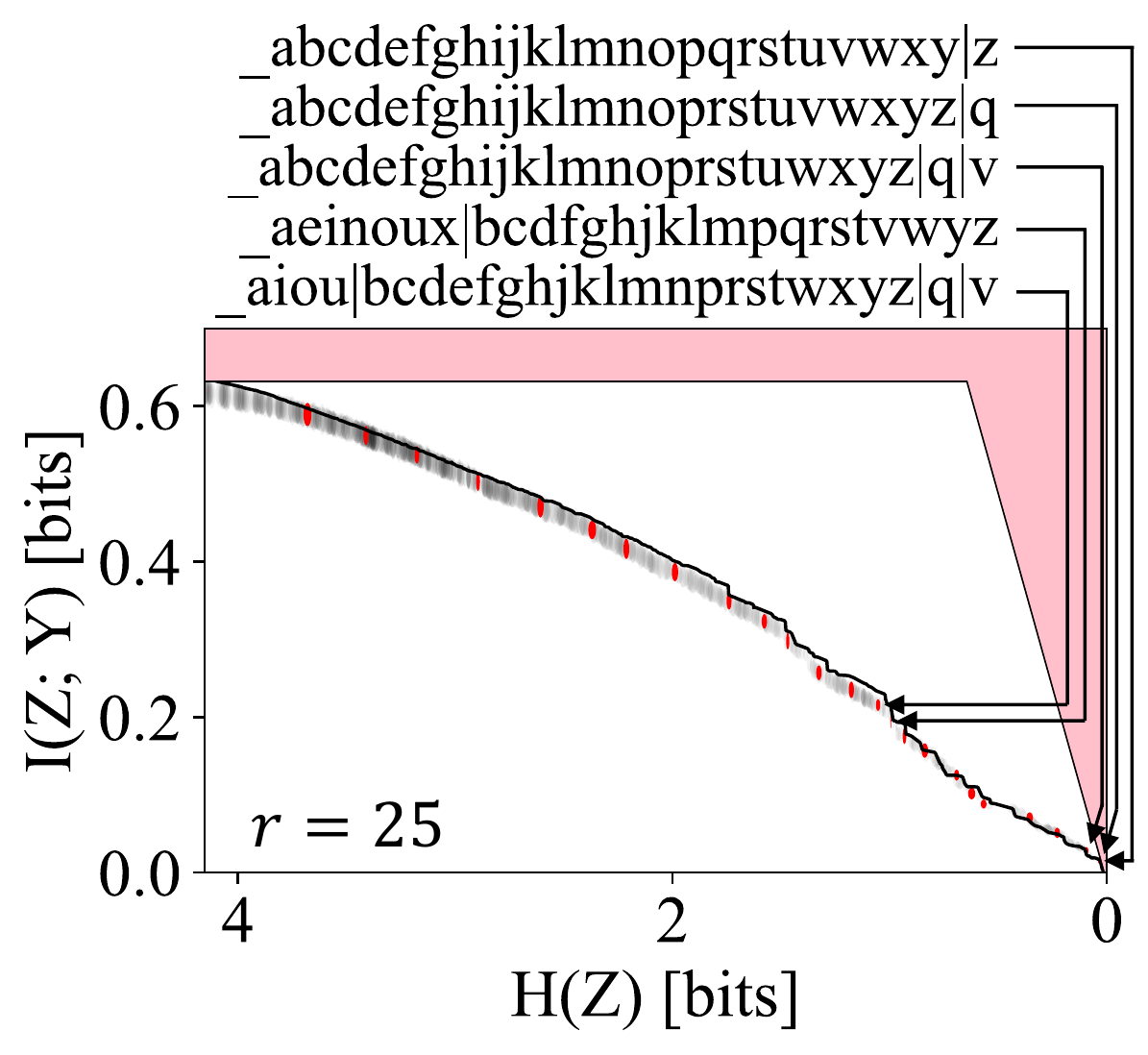}
                \caption{\label{fig:robust-maps-r25}}
            \end{subfigure}}
           \caption{The optimal frontier discovered by the Robust Pareto Mapper at various sampling ratios. 
           The points corresponding to robust clusterings selected by the algorithm are highlighted in red, with the rest in gray. 
           The true frontier is shown in solid black.
           }
           \label{fig:robust-maps}
    \end{figure}
    Finally, we would like to comment on choosing the parameter \(\epsilon\) in Algorithm~\ref{alg:mapper} and Algorithm~\ref{alg:robust-mapper} when working with limited sample sizes.
    The uncertainty in the frontier due to finite sampling effects naturally sets a scale for choosing \(\epsilon\).
    Ideally, we want the two length scales---that given by \(\epsilon\), and that due to the variance in the estimators---to be comparable.
    This ensures that we are not wasting resources fitting sampling noise.
    Evaluating the performance as a function of sample size and epsilon, we see that often, sample size is the limiting factor even up to significant sampling ratios, and often, a small \(\epsilon\) is often sufficient.
    This is demonstrated in Figure~\ref{fig:performance-sweeps}, where it can be seen that performance is good even with small \(\epsilon\), and increasing \(\epsilon\) does not result in a more accurate frontier until the sampling ratio is greater than \(r \approx 5 \times 10^4\).
    In practice, determining the appropriate \(\epsilon\) can be accomplished by selecting different holdout sets, and running the algorithm at a given \(\epsilon\) in each case;
    when \(\epsilon\) is chosen appropriately, the resulting Pareto frontier should not vary significantly between the runs.

    \begin{figure}[H] 
        \centering
         {\captionsetup{position=bottom,justification=centering}    \begin{subfigure}[b]{0.48\textwidth}
                \includegraphics[height=4.5cm]{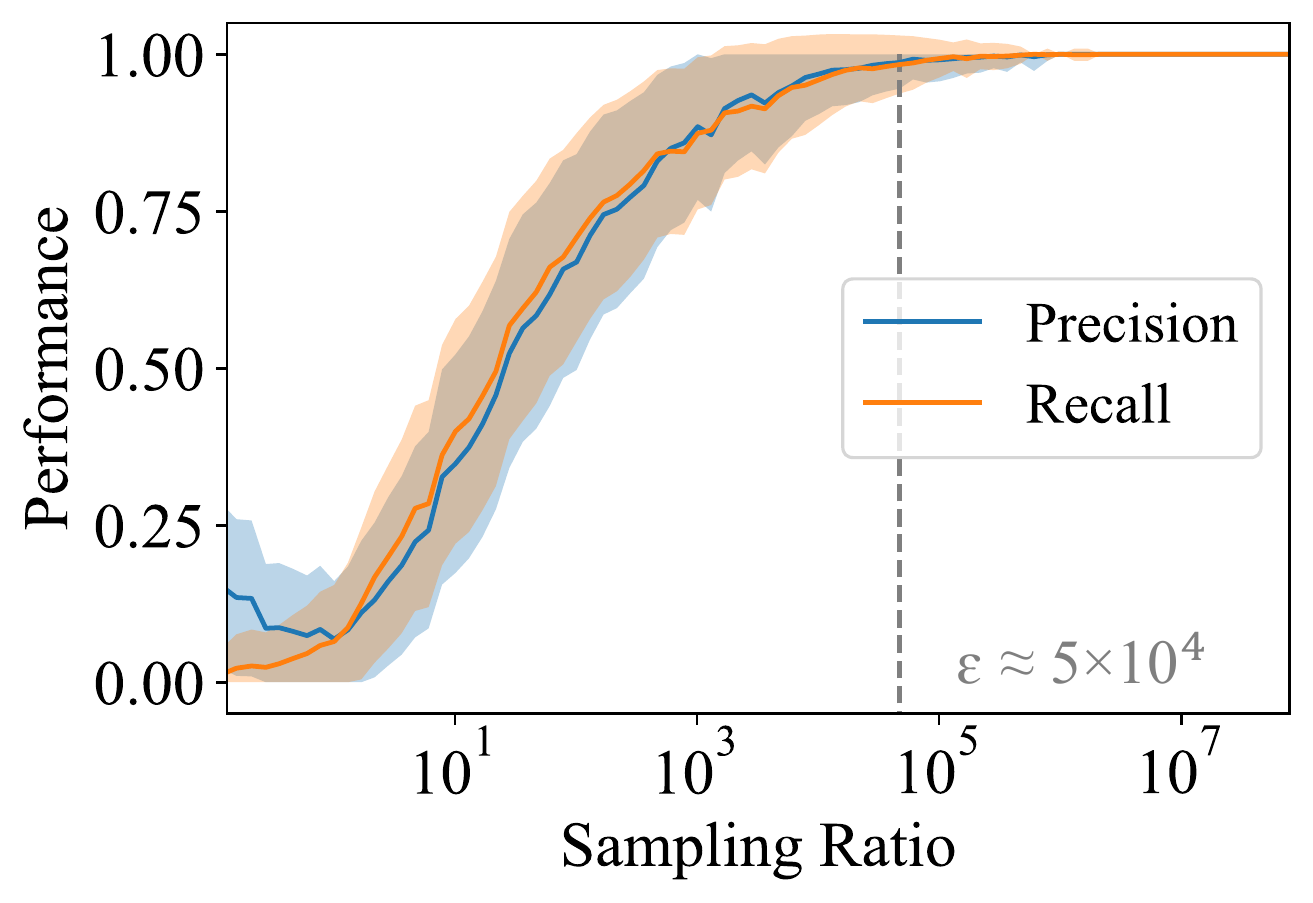}
                \caption{\label{fig:performance-ssweep}}
            \end{subfigure}
            ~
            \begin{subfigure}[b]{0.48\textwidth}
                \includegraphics[height=4.5cm]{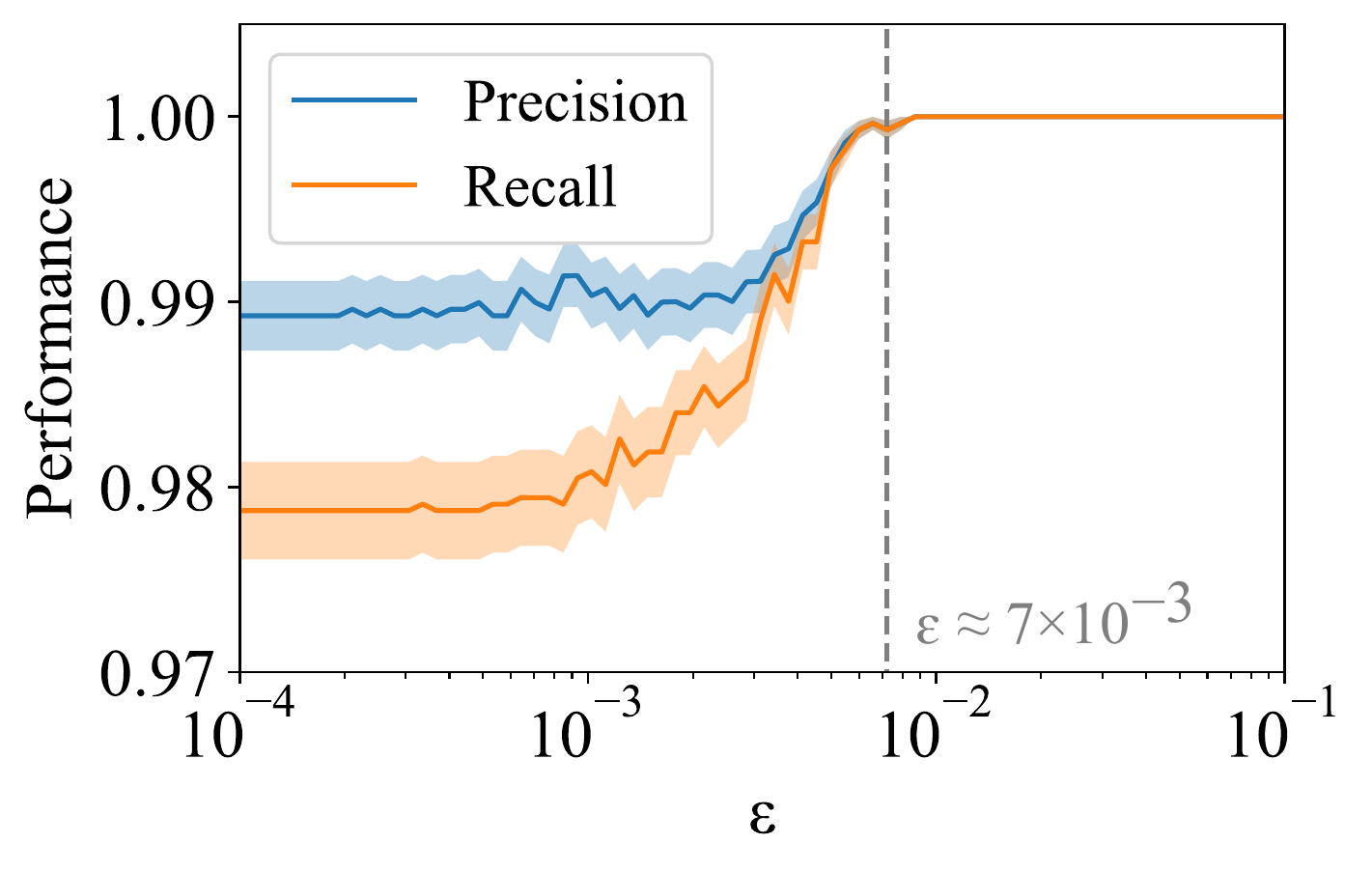}
                \caption{\label{fig:performance-esweep}}
            \end{subfigure}}
           \caption{Performance as a function of (\textbf{a}) sample size, and (\textbf{b}) \(\epsilon\). 
           The precision and recall are measured relative to the true frontier obtained by a brute force search on the true distribution.}
           \label{fig:performance-sweeps}
    \end{figure}

    \subsubsection{Naming the Colors of the Rainbow}
    \label{sec:color-dataset}

    Human languages often have a small set of colors to which words are assigned, and they remarkably often settle on similar linguistic partitions of the color space despite cultural and geographic boundaries \cite{TRBP21}.
    As our next example, we apply our method to the problem of optimally allocating words to colors based on the statistics of natural images.
    In order to cast this as a DIB-style learning problem, we consider the goal of being able to identify objects in natural images based solely on color:
    the variable we would like to predict, \(Y\), is therefore the class of the object (e.g., apple or banana).
    The variable we would like to compress, \(X\), is the average color of the object.
    The Pareto-optimal classifiers are those that, allotted limited memory for colorative adjectives, optimally draw the boundaries to accomplish the task of identifying objects.
    We demonstrate some success in discovering different color classes, relate it to those typically found in natural languages, and discuss shortcomings of our method.

    We create a dataset derived from the COCO dataset \cite{LMBH14}, which provides a corpus of segmented natural images with 91 object classes.
    There are a number of challenges we immediately face in the creation of this dataset, which require us to undertake a number of preprocessing steps.
    Firstly, using standard RGB color values, with 8 bits per channel, leaves over 16 million color classes to cluster, which is not feasible using our technique.
    Secondly, RGB values contain information that is not relevant to the task at hand, as they vary with lighting and image exposure.
    Thus, we turn to the HSV color model and use only the hue value (since hue is a circular quantity, we use circular statistics when discussing means and variances), which we refer to as the color of the object from now on.
    This leaves 256 values which are further reduced by contiguous binning so that each bin has roughly equal probability in order to maximize the entropy of \(X\).
    After this preprocessing, we are left with an input of size \(|X| = 30\).
    Another challenge we face is that there are often cues in addition to average color when performing object identification such as color variations, shape, or contextual understanding of the scene;
    in order to obtain the cleanest results, we retain only those classes that could reasonably be identified by color alone.
    Specifically, for the roughly 800,000 image segments from the approximately 100,000 images we considered in the COCO dataset, we calculate the average color of each segment and keep only the \(40\%\) with the most uniform color as measured by the variance of the hue across the segment;
    then, looking across classes, we keep only those that are relatively uniform on the average color of its instances, keeping approximately the most uniform \(20\%\) of classes.
    We are left with a dataset of approximately 80,000 objects across \(|Y| = 18\) classes, predictably including rather uniformly colored classes such as apples, bananas, and oranges.
    We chose these cutoff percentiles heuristically to maximize the predictive power of our dataset while maintaining a sizable number of examples.

    The Pareto frontier for this dataset is shown in Figure \ref{fig:color-frontier}.
    A number of DMC-optimal points are circled, and their respective color palettes are plotted below in descending order of likelihood.
    First, we note that the overall amount of relevant information is quite low, with a maximum \(I(X; Y) \approx 0.12\), indicating that despite our preprocessing efforts, color is not a strong predictor of object class.
    Unlike the other Pareto frontiers considered, there are a few prominent corners in this frontier, which is a sign that there is no clear number of colors to best resolve the spectrum.
    For the first few DMC-optimal clusterings, the colors fall broadly into reddish-purples, greens, and blues.
    This is somewhat consistent with the observation that languages with limited major color terms often settle for one describing warm colors and one describing colder colors \cite{TRBP21}.
    
    \begin{figure}[H]
        \centering
        \includegraphics[width=13.5cm]{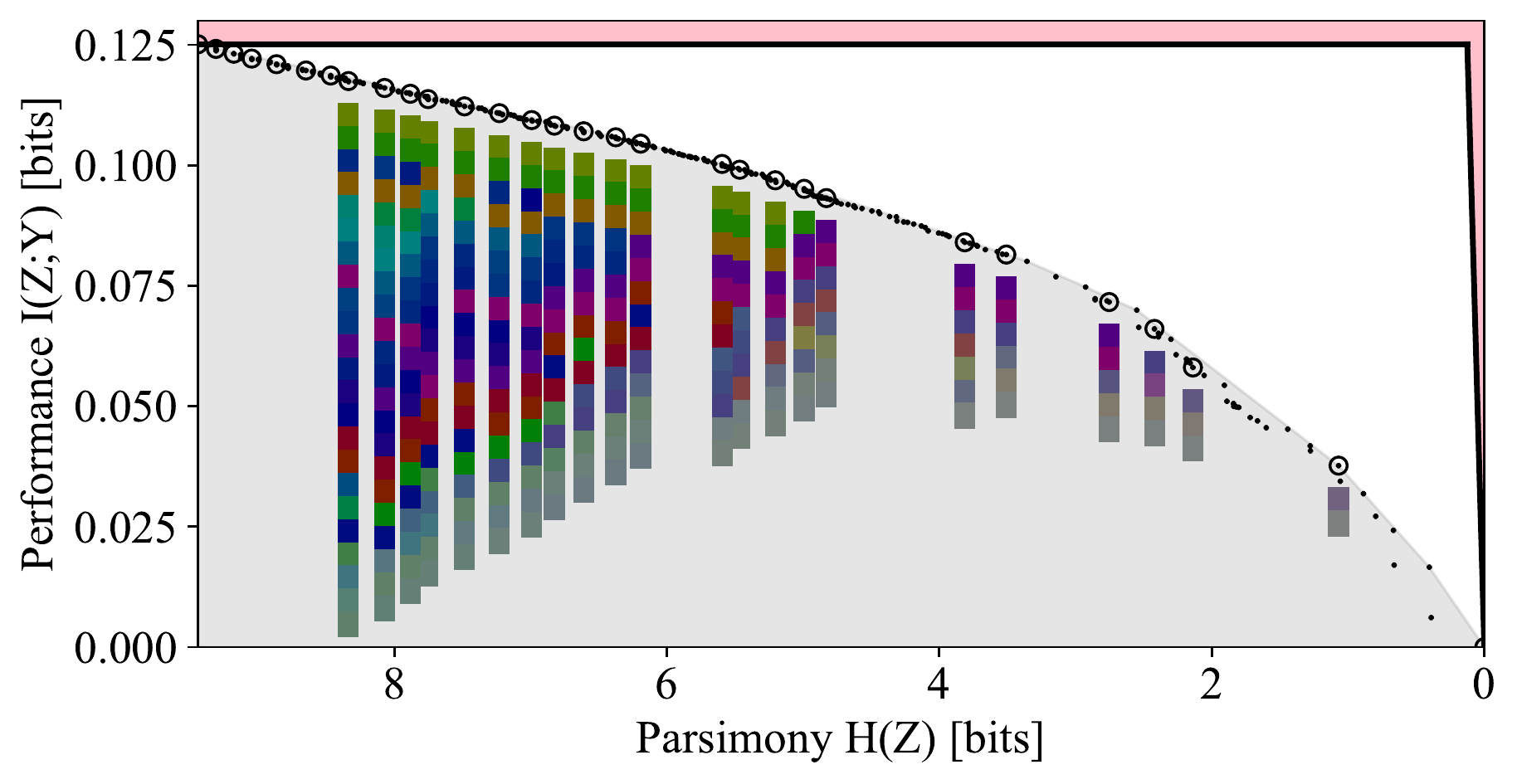}
        \caption{Pareto frontier of color data. 
        A representative color patch for each cluster is shown below select points sorted by likelihood. 
        The saturation of the patch represents the likelihood-weighted variance of the colors mapped to the class.}
        \label{fig:color-frontier}
    \end{figure}
 
    Overall, the results are not conclusive.
    We will address a few issues with our method and discuss how it might be improved.
    Firstly, as noted by \cite{TRBP21}, the colors present in human languages often reflect a communicative need and therefore should be expected to depend strongly on both the statistics of the images considered and also the prediction task at hand.
    Since the COCO dataset was not designed for the purpose of learning colors, classes had color outliers, despite our preprocessing efforts, which reduced the classification accuracy by color alone.
    Using color as a predictor of the variety of an apple or as a predictor of the ripeness of a banana might yield better results (see Figure \ref{fig:color-outliers}); indeed, these tasks might be more reflective of the communicative requirements under which some human languages developed \cite{TRBP21}.
    Due to the scarcity of relevant datasets, we have not attempted to address these subtleties.
   \begin{figure}[H]
        \centering
     {\captionsetup{position=bottom,justification=centering}    \begin{subfigure}[b]{0.22\textwidth}
            \centering
            \includegraphics[height=3.0cm]{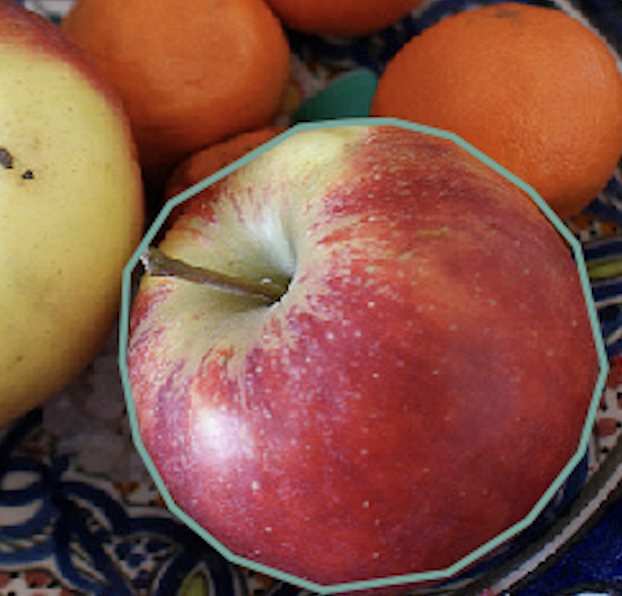}
            \caption{\vspace{0.2cm}}
        \end{subfigure}%
        ~ 
        \begin{subfigure}[b]{0.22\textwidth}
            \centering
            \includegraphics[height=3.0cm]{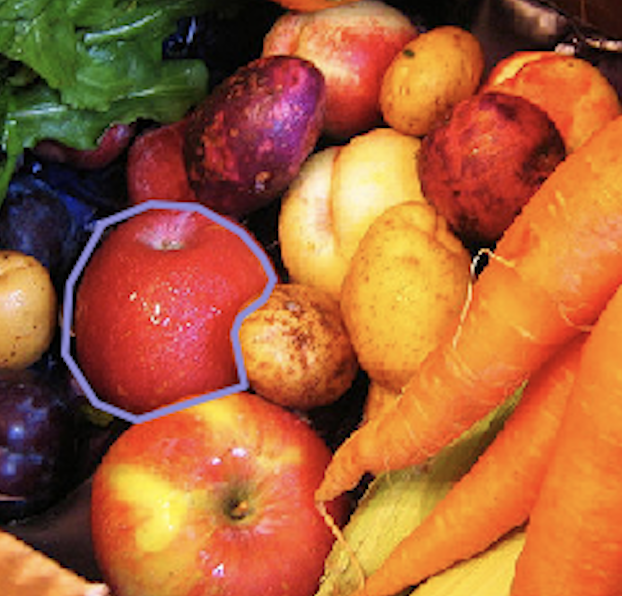}
            \caption{\vspace{0.2cm}}
        \end{subfigure}
        ~ 
        \begin{subfigure}[b]{0.22\textwidth}
            \centering
            \includegraphics[height=3.0cm]{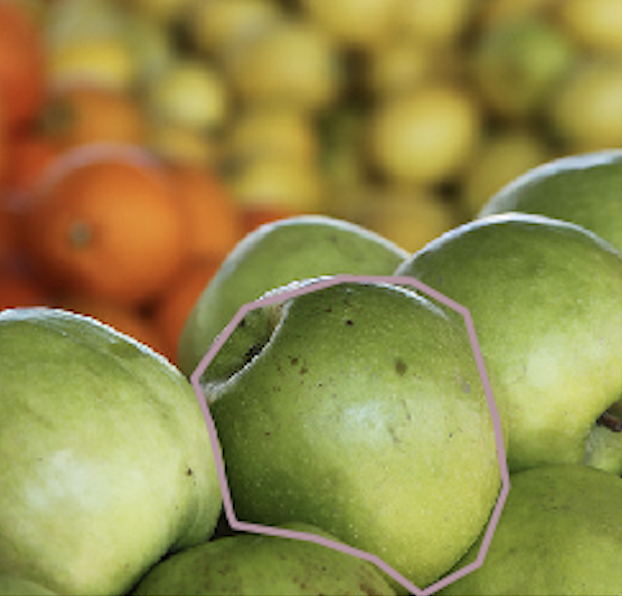}
            \caption{\vspace{0.2cm}}
        \end{subfigure}
         ~ 
        \begin{subfigure}[b]{0.22\textwidth}
            \centering
            \includegraphics[height=3.0cm]{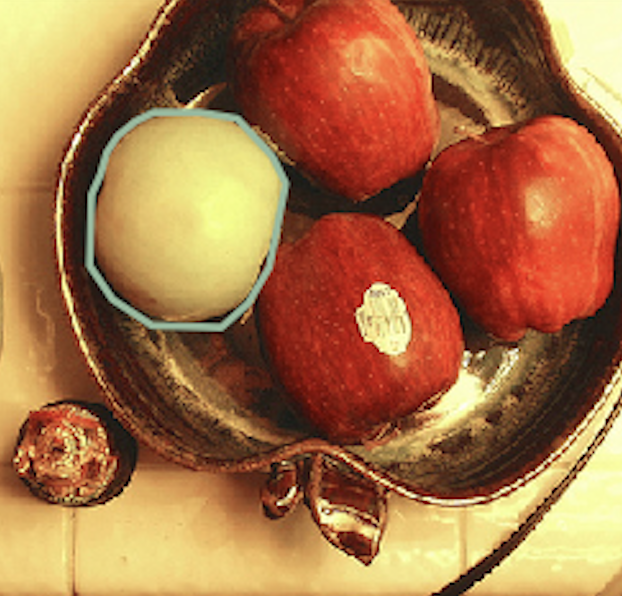}
            \caption{\vspace{0.2cm}}
        \end{subfigure}
        ~
        \begin{subfigure}[b]{0.22\textwidth}
            \centering
            \includegraphics[height=3.0cm]{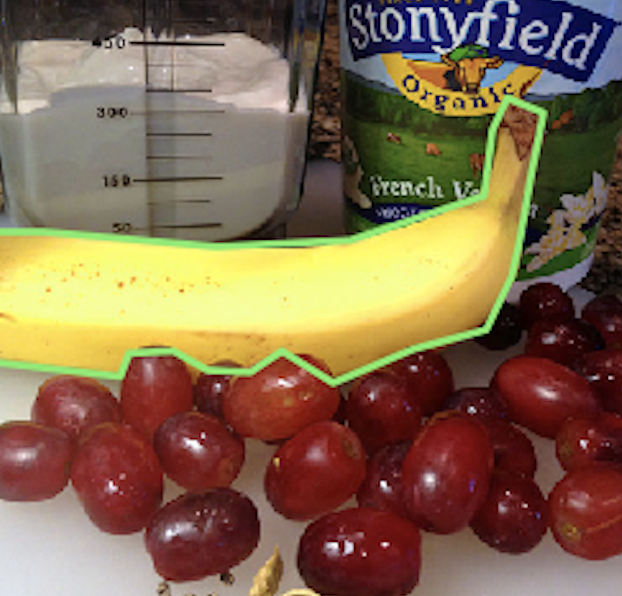}
            \caption{\vspace{0.2cm}}
        \end{subfigure}%
        ~ 
        \begin{subfigure}[b]{0.22\textwidth}
            \centering
            \includegraphics[height=3.0cm]{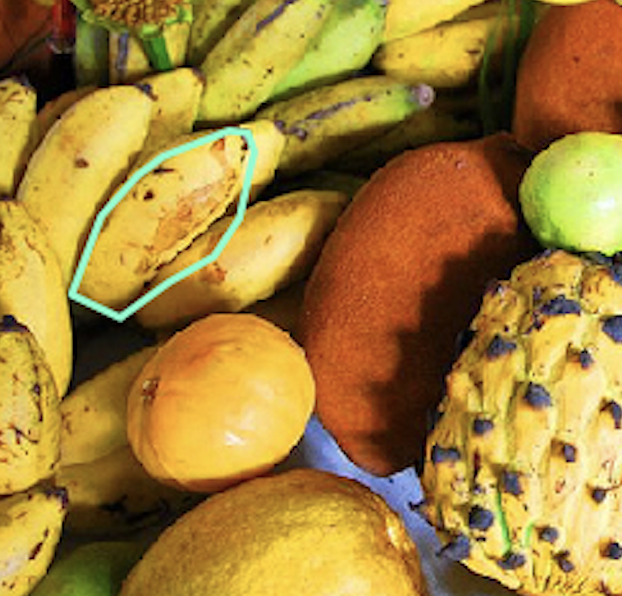}
            \caption{\vspace{0.2cm}}
        \end{subfigure}
        ~ 
        \begin{subfigure}[b]{0.22\textwidth}
            \centering
            \includegraphics[height=3.0cm]{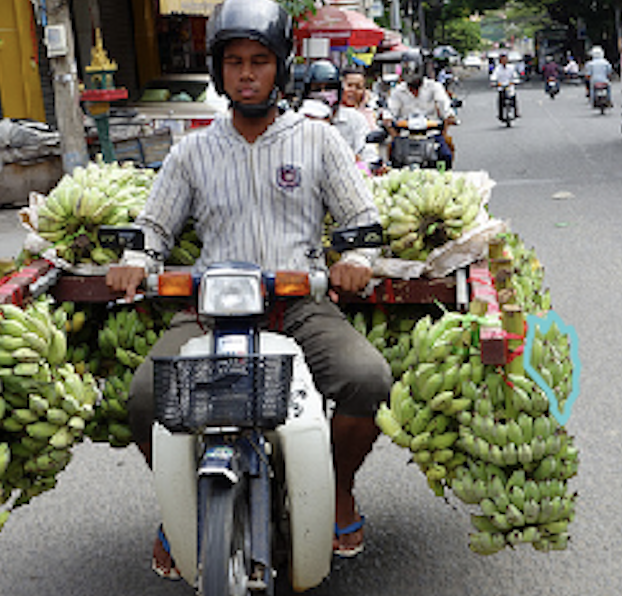}
            \caption{\vspace{0.2cm}}
        \end{subfigure}
         ~ 
        \begin{subfigure}[b]{0.22\textwidth}
            \centering
            \includegraphics[height=3.0cm]{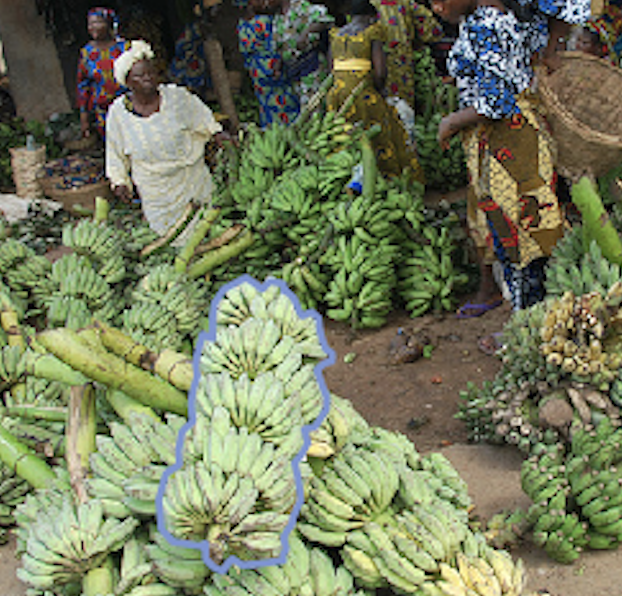}
            \caption{\vspace{0.2cm}}
        \end{subfigure}}
        \caption{Examples of correctly (\textbf{a},\textbf{b}) and incorrectly (\textbf{c},\textbf{d}) identified apples; and correctly (\textbf{e},\textbf{f}) and incorrectly (\textbf{g},\textbf{h}) identified bananas from the filtered COCO dataset using the best discovered five-bin clustering.}
        \label{fig:color-outliers}
    \end{figure}
    Another issue, more fundamental to the DIB algorithm, is that DIB is not well suited for the compression of domains of a continuous nature.
    The DIB trade-off naturally favors a discrete domain, \(X\), without a measure of similarity between objects in \(X\).
    Unlike the other examples considered, the space of colors is inherently continuous: there is some notion of similarity between different hues.
    One weaknesses of the DIB trade-off is that it does not respect this natural notion of closeness and it is as likely to map distant hues together as it is ones that are close together.
    This is undesirable in the case of the color dataset, as we would ideally like to map contiguous portions of the color space to the same output.
    Other objectives, such as the IB or a multidimensional generalization of \cite{TW19}, may be more suitable in cases where the domain is of a continuous nature.

    \subsubsection{Symmetric Compression of Groups}
    \label{sec:group-dataset}

    For our final example, we turn our attention to a group--theoretic toy example illustrating a variation on the compression algorithm so far considered which we call ``symmetric compression.''
    We consider a triplet of random variables \((X_1, X_2, Y)\), each taking on values in the set \(G\) with the special property that \(G\) forms a group under the binary group operation `\(\cdot\)'.
    We could apply Algorithm~\ref{alg:mapper} directly to this problem by setting \(X = (X_1, X_2)\), but this is not ideal, as it does not make use of the structure we know the data to have and as a result needlessly expands our search space.
    Instead, we make the slight modification, detailed in Appendix~\ref{app:symmetric-mapper}, where we apply the same clustering to both inputs, \(Z = (f(A), f(B))\).
    We would like to discover an encoding \(f\) that trades off the entropy of the encoding with the ability to predict \(Y\) from \((f(X_1), f(X_2))\).
    We expect that the DIB frontier encodes information about the subgroups of the group \(G\), but we also expect to find points on the frontier corresponding to near-subgroups of \(G\).

    We consider two distributions. 
    The first consists of the sixteen integers that are co-prime to 40, i.e., $\{1, 3, 7, 9, 11, 13, 17, 19, 21, 23, 27, 29, 31, 33, 37, 39\}$, 
    which for a multiplicative group module 40 denoted \((\mathbb{Z}/40\mathbb{Z})^\times\).
    The second is the Pauli group \(G_1\), whose elements are the sixteen $2\times 2$ matrices generated by the Pauli matrices under matrix multiplication: they are the identity matrix $I$ and Pauli matrices $X,Y,Z$, each multiplied by $\pm 1$ and $\pm i$.
    These groups are chosen as they both have order \(16\) but are otherwise quite different;
    notably, \((\mathbb{Z}/40\mathbb{Z})^\times\) is abelian while \(G_1\) is not.
    The joint probability distribution is defined as follows for each group \(G\):
    we take \((X_1, X_2)\) to be distributed uniformly over \(G^2\) and \(Y = X_1 \cdot X_2\).
    The distribution \(p_{X_1 X_2 Y}\) is given as input to the symmetric Pareto Mapper (Algorithm~\ref{alg:symmetric-mapper}).

    The resultant frontiers are shown in Figure~\ref{fig:group-frontier}.
    As expected, the subgroups are readily identified in both cases, as seen the in circled points on the frontier with entropy \(H(Z) = 1\), \(H(Z) = 2\), and \(H(Z) = 3\), corresponding to subgroups of size \(2\), \(4\), and \(8\), respectively.
    In this example, we also see that the clusterings corresponding to the subgroups saturate the feasibility bound of \(I(Z; Y) = H(Z)\), indicating that at these points, all the information captured in \(Z\) is relevant to \(Y\).
    At these points, the encoding effectively identifies a subgroup \(H \le G\) and retains information only about which of the \(| G | / | H |\) cosets an element belongs to; as it retains the identity of the cosets of \(X_1\) and \(X_2\) in \(Z_1\) and \(Z_2\), it is able to identify the coset of the output \(Y\), thereby specifying \(Y\) to \(\log_2 \frac{| G |}{| H |}\) bits.
    These clearly desirable solutions show up prominently in the primal DIB frontier, yet their prominence is not evident on the frontier of the Lagrangian DIB---notably having zero kink angle as defined by \cite{SS17a}.

    In addition to the points corresponding to identified subgroups, a number of intermediary points have also been highlighted showing `near-subgroups', where, allotted a slightly larger entropy budget, the encoder can further split cosets apart in such a way that partial information is retained.
    Interestingly, despite being very different groups, they have identical Pareto frontiers.
    This is because they both have subgroups of the same cardinality, and the entropy and relevant information of these encodings is agnostic to the group theoretic details and concerns itself only with the ability to predict the result of the group operation.
    \begin{figure}[H]       
        \centering
          {\captionsetup{position=bottom,justification=centering}   \begin{subfigure}[b]{0.5\textwidth}
                \centering
                \includegraphics[width=7.5cm]{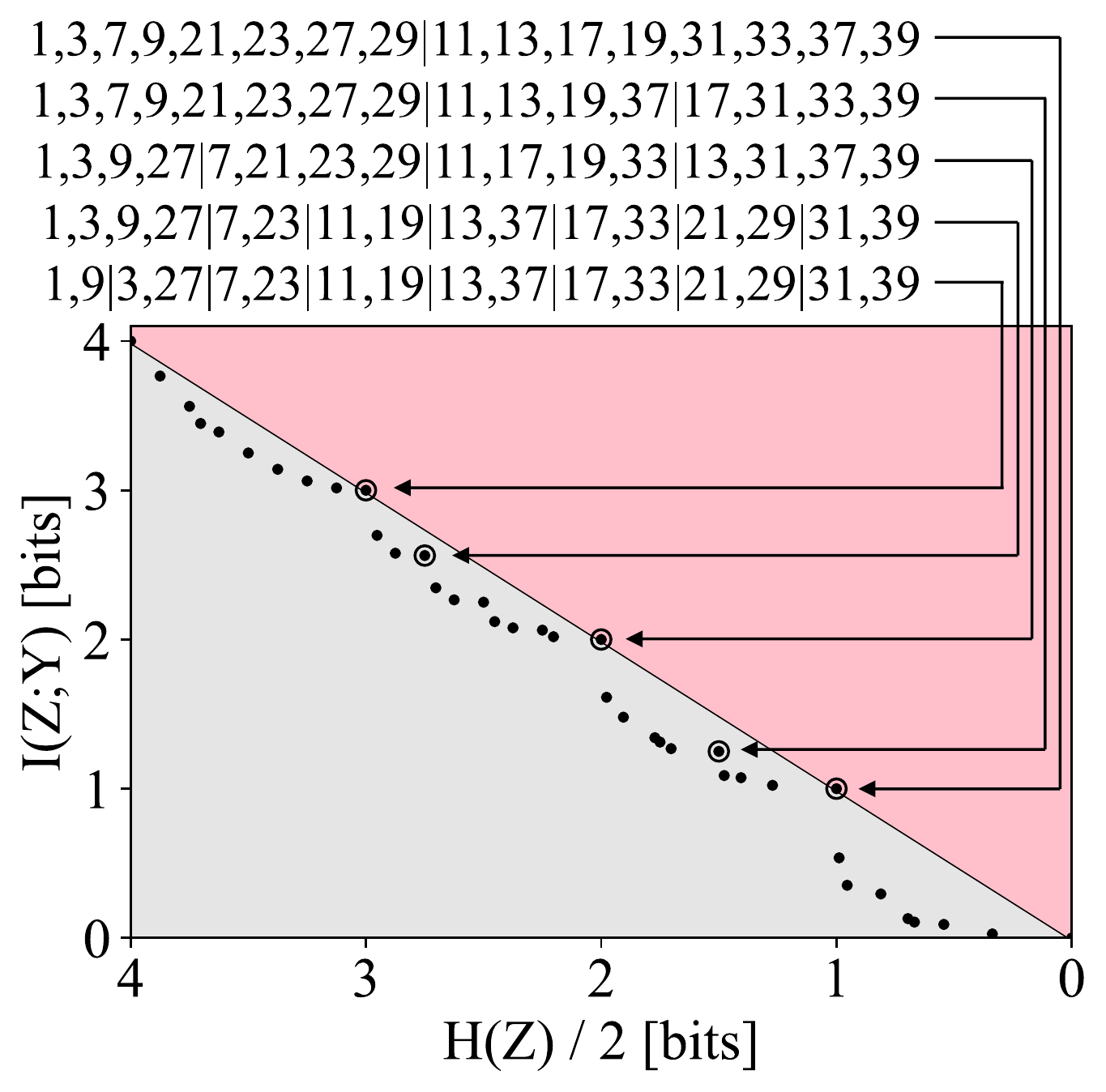}
                \caption{}
            \end{subfigure}%
            ~ 
            \begin{subfigure}[b]{0.5\textwidth}
                \centering
                \includegraphics[width=7.5cm]{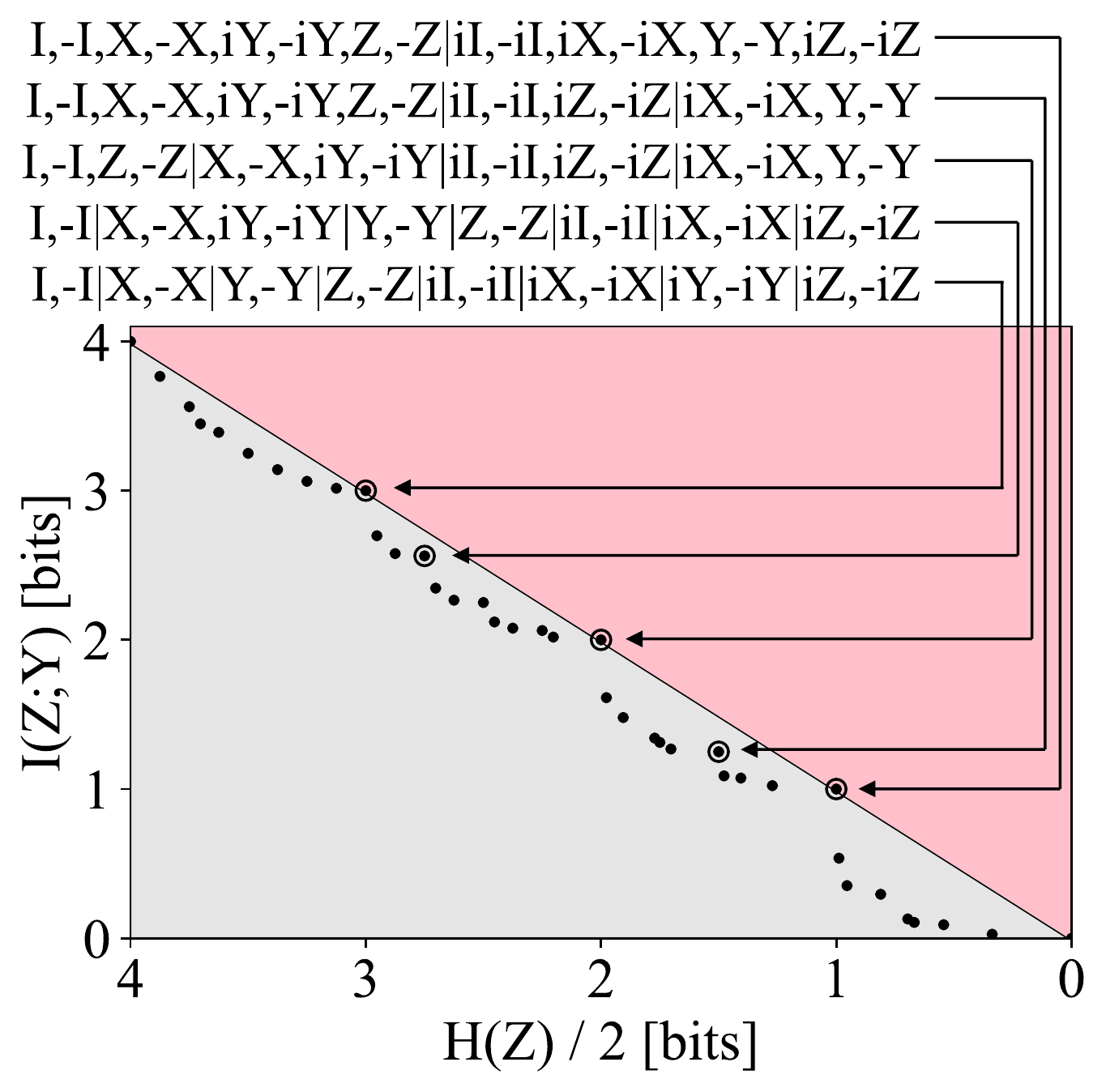}
                \caption{}
            \end{subfigure}}
            \caption{Discovered frontier of the (\textbf{a}) \(\left(\mathbb{Z}/40\mathbb{Z}\right)^\times\) group and (\textbf{b}) the non-abelian Pauli group. 
            Both groups have identical frontiers despite having different group structures.
            }
            \label{fig:group-frontier}
    \end{figure} 


\section{Discussion}
    \label{sec:discussion}
    We have presented the Pareto Mapper algorithm, which computes the optimal trade-off between parsimony and performance for lossy data compression. 
    By applying it to examples involving linguistics, image colors and group theory,
    we have demonstrated the richness of the DIB Pareto frontier that customarily lies hidden beneath the convex hull.
    Our English alphabet example revealed features at multiple scales and examples of what the frontier structure reveals about the data,  and we demonstrated a modification to our algorithm that can aid model selection given significant sampling noise.
    Notably, we showed how the prominence of a point on the primal frontier can be a sharper tool for model selection than existing measures on the Lagrangian DIB frontier; for example, for our group theory examples, it outperformed the kink-angle method for model selection, which only gave kink angles of zero.
    Our datasets and implementation of the presented methods are freely available on GitHub ({\url{https://github.com/andrewktan/pareto_dib}}).

    Our result helps shed light on recently observed phases transitions.
    Recent work has shown that learning phase transitions can occur when optimizing two-objective trade-offs including the (D)IB \cite{AS18, WF20, WFCT20, NS21} and \(\beta\)-VAEs \cite{RV18}.
    In these cases, it is found that the performance of the learned representation makes discontinuous jumps as the trade-off parameter \(\beta\) is varied continuously.
    Such phase transitions can be readily understood in terms of the primal Pareto frontier of the trade-off:
    methods that optimize the Lagrangian DIB are only able to capture solutions on the convex hull of objective plane;
    as the Pareto frontier is largely convex, methods that optimize the Lagrangian exhibit will discontinuous jumps when the trade-off parameter \(\beta\) (which corresponds to the slope of a tangent to the frontier) is varied. 
    This is analogous to the way first-order phase transitions in statistical physics arise, where it is the closely related Legendre--Fenchel dual that is minimized.

    We would like to emphasize that, going beyond the IB framework, our basic method (Section~\ref{sec:pareto-mapper}) is generally applicable to a large class of two-objective optimization problems, including general clustering problems.
    Specifically, our method can be adapted for two-objective trade-offs with the following properties: a discrete search space; a frontier that, for typical datasets, grows polynomially with the input size \(|X|\); and a notion of relatedness between objects in the search space (e.g., for the DIB problem, new encodings can be derived from existing ones by merging its output classes), which allows for an agglomerative search.
    The modification (Section~\ref{sec:robust-pareto-mapper}) can also be adapted given suitable estimators for other two-objective trade-offs.

    \subsection*{Outlook}
    There are many opportunities to further improve our results both conceptually and practically.
    To overcome the limitations we highlighted with our image color dataset, it will be interesting to generalize our work and \cite{TW19} to compressing continuous variables potentially with trade-offs such as the IB.
    While our evidence for the polynomial scaling of the size of the Pareto frontier is likewise applicable to other trade-offs of this sort, the runtime of our algorithm depends heavily on how quickly the search space can be pruned away and therefore is not guaranteed to be polynomial.
    Here, there is ample opportunity to tighten our analysis of the algorithmic complexity of finding the DIB frontier and on the scaling of generic Pareto frontiers.

    Proofs aside, it will also be interesting to optimize the algorithm runtime beyond simply showing that it is polynomial.
    Although we have demonstrated the polynomial scaling of our algorithm for realistic datasets, the polynomial is of a high degree for our implementation, placing limits of \(|X| \le 50\) in practice.
    There are fundamental lower bounds on the runtime set by the scaling of the Pareto set, which we have shown in Figure~\ref{fig4}b to be approximately \(O(n^{2.1})\) for realistic datasets; 
    however, there is likely to be some room for reducing the runtime by sampling clusterings from a better distribution.
    Another opportunity for improvement is increasing the speed at which a given point can be evaluated on the objective plane, which is evidenced by the gap between the runtime, approximately \(O(n^{5.0})\), and the number of points searched, \(O(n^{3.0})\) (Figure~\ref{fig:scaling-random-data}a).
    
    While our method is only applicable to trade-offs over discrete search spaces, the Pareto frontier over continuous search spaces can also fail to be (strictly) concave.
    For example, the inability for the Lagrangian formulation of the (D)IB to explore all points on the trade-off has previously been studied in \cite{KTV18}.
    They propose a modification to the (D)IB Lagrangian that allows for the exploration of parts of the frontier that are not strictly concave.
    An interesting direction for future work is to study whether a similar modification to the Lagrangian can be used to discover the convex portions of similar trade-offs, including those over discrete spaces.
    Another direction for future work is to compare the primal DIB frontier with solutions to the IB;
    while solutions to the DIB Lagrangian often perform well on IB plane \cite{SS17b}, it is an open question whether the solutions to the primal DIB perform favorably.
    Finally, as pointed out by a helpful reviewer, the dual problem corresponding to the primal problem of Equation~\eqref{eq:primalopt}, being a convex optimization problem, is also an interesting direction for future study.
    
    We would also like to note that Pareto-pruned agglomerative search is a generic strategy for mapping the Pareto frontiers of two-objective optimization problems over discrete domains.
    The Pareto Mapper algorithm can also be extended to work in multi-objective settings given an appropriate implementation Pareto set in higher dimensions.
    We conjecture that the poly-logarithmic scaling of the Pareto set holds in higher dimensions as well.
    Extending this work to multi-objective optimization problems is another interesting direction for future work.
    
    In summary, multi-objective optimization problems over discrete search spaces arise naturally in many fields from machine learning \cite{SB04,TZ15,AFDM16,SS17a,SS17b,CS18,SBDA19} to thermodynamics \cite{S20} and neuroscience \cite{BM10}.
    There will therefore be a multitude of interesting use cases for further improved techniques that map these Pareto frontiers in full detail, including concave parts that reveal interesting structure in the data.

\vspace{6pt}

\section{Acknowledgements}
    We thank Ian Fischer for helpful discussions. 
    We would also like to thank the anonymous reviewers for helpful clarifications and suggestions for future work.
    This work was supported by the Institute for Artificial Intelligence and Fundamental Interactions (IAIFI) through NSF Grant No. PHY-2019786 and the Natural Sciences, the Engineering Research Council of Canada (NSERC) [PGSD3-545841-2020], the Casey and Family Foundation, the Foundational Questions Institute and the Rothberg Family Fund for Cognitive Science.



\bibliographystyle{unsrt}
\bibliography{refs}


\clearpage
\appendix
\section{Proof of Pareto Set Scaling Theorem}
    \label{app:pareto-size}

    As discussed in Section~\ref{sec:pareto-properties}, the performance of our algorithm depends on the size of the Pareto frontier.
    In the paper, we provide experimental evidence for the polynomial scaling of the DIB Pareto frontier of a variety of datasets.
    In this appendix, we will prove Theorem~\ref{thm:scaling-by-area}, which provides sufficient conditions for the sparsity of the Pareto frontier and apply it to a number of examples.

    As in Section~\ref{sec:pareto-properties}, let \(S = \{(U_i, V_i)\}_{i=1}^N\) be a sample of \(N\) i.i.d. bivariate random variables having joint cumulative distribution \(F_{UV}(u, v)\).
    Further, let \(R_{S, U}(U_i)\) and \(R_{S, V}(V_i)\) be the marginal rank statistics of \(U\) and \(V\), respectively, with respect to \(S\);
    that is, \(U_i\) is the \(R_{S, U}(U_i)\)\textsuperscript{th} smallest \(U\)-value in \(S\) and likewise for \(V\).
    Ties can be broken arbitrarily.
    We will often drop the subscripts on \(R_{S, U}\) and \(R_{S, V}\) when the context is obvious.

    \begin{definition}
        Given a permutation \(\sigma: [N] \rightarrow [N]\) where \([N] \equiv \{1, \ldots, N\}\), we call \(i\) a \textit{sequential minimum} if \(j < i \Rightarrow \sigma(j) > \sigma(i)\).
    \end{definition}

    We would now like to show that the marginal rank statistics \(S\) are sufficient for determining membership in \(\operatorname{Pareto}(S),\) which we formalize in Lemma~\ref{lem:rank-suff}.

    \begin{lemma}
        Let \(\sigma_U(i) = R(U_i)\) and \(\sigma_V(i) = R(V_i)\).
        An element \((U_i, V_i) \in S\) is maximal if and only if its rank, \(i\), is a sequential minimum of \(\sigma_U \circ \sigma_V\).
        \label{lem:rank-suff}
    \end{lemma} 

    \begin{proof}
        (\(\implies\)) 
        Assume \((U_i, V_{\sigma(i)}) \in S\) is maximal. 
        For any other point \((u_j, v_{\sigma(j)}) \in S\), \(i \ne j\), if \(j < i \Rightarrow u_i < u_j,\) then \(v_{\sigma(i)} > v_{\sigma(j)}\) by definition of maximality, which implies \(\sigma(j) > \sigma(i),\) showing that \(i\) is a sequential minimum of \(\sigma\). 

        (\(\impliedby\)) For \((u_i, v_{\sigma(i)}) \in S\) such that \(i\) is a sequential minimum of \(\sigma\). 
        For any other point \((u_j, v_{\sigma(j)}) \in S\), \(i \ne j\), either \(i < j \Rightarrow u_i > u_j\) showing that \((u_i, v_{\sigma(i)})\) is maximal, 
        or \(j > i \Rightarrow \sigma(j) > \sigma(i)\) by definition of a sequential minimum, which implies \(v_{\sigma(i)} > v_{\sigma(j)}\) showing that \((u_i, v_{\sigma(i)})\) is maximal.
    \end{proof}

    \begin{corollary}
        Membership in the Pareto set is invariant under strictly monotonic transformations of \(U\) or \(V\).
        \label{cor:monotone-transformation}
    \end{corollary}

    \begin{proof}
        Strictly monotonic transformations leave the rank statistics unchanged and therefore also do not affect membership in the Pareto set by Lemma~\ref{lem:rank-suff}.
    \end{proof}

    We now turn to the main result of this Appendix: the proof of Theorem~\ref{thm:scaling-by-area}, which is restated here for {convenience}.

    \scalingbyarea*

    \begin{proof}
        Since the marginal CDFs are invertible by assumption and therefore strictly monotonic, Corollary~\ref{cor:monotone-transformation} allows us to consider instead \(U_i' = F_U(U_i)\) and \(V_i' = F_V(V_i)\) with the promise that \(\operatorname{Pareto}(S') = \operatorname{Pareto}(S)\) where \(S' \equiv \{(U_i', V_i'\}\).
        Note that \(F_{U'}(u') = u'\) and \(F_{V'}(v') = v'\), and therefore without loss of generality, we can assume \(F_U\) and \(F_V\) are uniform distributions over the interval \([0, 1]\) dropping the prime notation.
        This allows us to identify the copula with the joint CDF \(C(F_U(u), F_V(v)) = C(u, v) = F(u, v)\).

        Let \(\mathbf{1}_{A}(x)\) denote the indicator function of a set \(A\):
        taking the value \(1\) for \(x \in A\) and \(0\) otherwise.
        Then, \(\E_{S}\left[|\operatorname{Pareto}(S)|\right] = \E_{S} \left[\sum_{i = 1}^N \mathbf{1}_{\operatorname{Pareto}(S)}(U_i, V_i)\right]\).
        Making use of the linearity of expectation and noting that \((U_i, V_i)\) are drawn i.i.d., we can write
        \begin{equation}
            \E_{S}\left[|\operatorname{Pareto}(S)|\right] = N \E_{S} \left[\mathbf{1}_{\operatorname{Pareto}(S)}(U_1, V_1)\right]
        \end{equation}
        Note that \(\E\left[\mathbf{1}_{\operatorname{Pareto}(S)}(u, v)\right] = (1 - \operatorname{Pr}[U > u, V > v])^N = (u + v - C(u, v))^N\), which follows from the definition of Pareto optimality.
        For convenience, we define \(\hat{C}(u, v) \equiv u + v - C(u, v)\) yielding
        \begin{equation}
            \E_{S}\left[|\operatorname{Pareto}(S)|\right] = \int_0^1 \int_0^1 N f(u, v) \hat{C}(u, v)^{N - 1} du dv
        \end{equation}
        Take \(f_{\mathrm{\max}}\) to be the maximum value \(f\) achieves over the domain, we are guaranteed \(f_{\mathrm{max}} < \infty\) as \(C\) is Lipschitz by assumption. 
        Therefore
        \begin{equation}
            \E_{S}\left[|\operatorname{Pareto}(S)|\right] \le N f_{max} \int_0^1 \int_0^1 \hat{C}(u, v)^{N - 1} du dv
        \end{equation}
        Now, define \(\hat{C}_N\), which is equal to \(\hat{C}\) in the region \(R_N\) and \(0\) otherwise.
        We also define the region
        \begin{equation}
            R_N' \equiv \left\{(u, v) \in [0, 1] \times [0, 1] : e^{-\frac{1 + 2 \log N}{N}} \le \hat{C}(u, v) < e^{-\frac{1}{N}}\right\}
        \end{equation}
        We now split the integral over \([0, 1]^2\) into three disjoint parts
        \begin{equation}
            \int_0^1 \int_0^1 \hat{C}(u, v)^{N - 1} du dv
            =
            \int_{R_N} \hat{C}_N(u, v)^{N - 1} du dv 
            + \int_{R_N'} \hat{C}(u, v)^{N - 1} du dv 
            + \int_{[0, 1]^2 \setminus R_N \cup R_N'} \hat{C}(u, v)^{N - 1} du dv 
        \label{eq:truncated-copula-convergence}
        \end{equation}
        The integrand of the final term is bounded by \(e^{-\log (N) + O(1)} = O(N^{-1})\) and \(\lambda([0, 1]^2 \setminus R_N \cup R_N') = \Theta(1)\); therefore, this term goes to \(0\) as \(N \rightarrow \infty\).
        Now, we turn to the middle term on the right-hand side.
        Since \(C\) is \(2\)-non-decreasing and Lipschitz, we have that the measure of the set \(\lambda(R_N') = \Theta\left( e^{-\frac{1}{N}} - e^{-\frac{1 + 2 \log N}{N}} \right) = \Theta\left(\frac{\log N}{N}\right)\),
        \(\hat{C}(u, v) < e^{-\frac{1}{N}}\) in the region \(R_N'\) by definition, and therefore, the second term goes to \(0\) as \(N \rightarrow \infty\).
        Since there is always at least one point on the Pareto frontier, the first term must be \(\Omega(1)\), and the integral is dominated by the portion over \(R_N\).
        Equivalently, 
        \begin{equation}
            \int_0^1 \int_0^1 \hat{C}(u, v)^{N - 1} du dv
            \sim
            \int_0^1 \int_0^1 \hat{C}_N(u, v)^{N - 1} du dv 
        \end{equation}
        Further,
        \begin{equation}
            \int_0^1 \int_0^1 N C_N(u, v)^{N - 1} du dv
            \le 
            N \int_0^1 \int_0^1 \mathbf{1}_{R_N}(u, v) du dv = N \lambda(R_N) = \ell(N)
        \end{equation}
        Following the chain of inequalities and asymptotic equivalences, we arrive at the desired result \(\E_{S}\left[|\operatorname{Pareto}(S)|\right] = \Theta(\ell(N))\).
    \end{proof}

    We now apply Theorem~\ref{thm:scaling-by-area} to a few illustrative examples.

    The Fr\'echet--Hoeffding copulae, \(W\) and \(M\), are extremal in the sense that, written in \(two\) dimensions, any copula \(C\) must satisfy \(W(u, v) \le C(u, v) \le M(u, v)\), \(\forall (u, v) \in [0, 1]^2\);
    where \(W(u, v) = \max(u + v - 1, 0)\) and \(M(u, v) = \min(u, v)\).
    \(W\) and \(M\) correspond to complete negative and positive monotonic dependence, respectively. 

    \begin{example}[Fr\'echet--Hoeffding lower bound]
    First, let us consider the scaling of the Pareto of a distribution with extremal copula \(W(u, v)\).
    In this case, we note that the region \([0, 1]^2 \setminus R_N\) is the triangle with vertices at \(\{(0, 0), (0, e^{-1/N}), (e^{-1/N}, 0)\}\), and therefore \(\lambda(R_N) = 1 - \frac{1}{2} \exp^{-\frac{2}{N}}\).
    For large \(N\), \(\lambda(R_N) = \frac{1}{2} + O(N^{-1})\).
    We see that this satisfies the conditions for Theorem~\ref{thm:scaling-by-area} with \(\ell(N) = N\), giving \(\E_{S}\left[|\operatorname{Pareto}(S)|\right] = \Theta(N)\) as expected for a distribution with complete negative monotonic dependence.
    \end{example}

    \begin{example}[Fr\'echet--Hoeffding upper bound]
    First, let us consider the scaling of the Pareto of a distribution with extremal copula \(M(u, v)\).
    In this case, we note that the region \([0, 1]^2 \setminus R_N\) is the region \([0, e^{-1/N}]\), and therefore, \(\lambda(R_N) = 1 - \exp^{-\frac{2}{N}}\).
    For large \(N\), \(\lambda(R_N) = \frac{2}{N} + O(N^{-2})\).
    We see that this satisfies the conditions for Theorem~\ref{thm:scaling-by-area} with \(\ell(N) = 1\), giving \(\E_{S}\left[|\operatorname{Pareto}(S)|\right] = \Theta(1)\) as expected for a distribution with complete positive monotonic dependence.
    \end{example}

    \begin{example}[Independent random variables]
    Next, let us consider the case of independent random variables with copula \(C(u, v) = uv\).
    Note that the level curves in this case \(u + v - C(u, v) = e^{-\frac{1}{N}}\) are given by \(v = \frac{e^{-\frac{1}{N}} - u}{1 - u}\).
    We can then integrate to find the area of the region \(R_N\)
    \begin{equation}
        \lambda(R_N) = 1 - \int_{0}^{e^{-\frac{1}{N}}} \frac{e^{-\frac{1}{N}} - u}{1 - u} du = e^{-1/n} \left(1 - e^{1/n}\right) \left(\log \left(1-e^{-1/n}\right)-1\right)
    \end{equation}
    Expanding for large \(N\), we find that \(\lambda(R_N) = \frac{\log N}{N} + O(N^{-1})\).
    We see that this satisfies the conditions for Theorem~\ref{thm:scaling-by-area} with \(\ell(N) = \log N\), giving \(\E_{S}\left[|\operatorname{Pareto}(S)|\right] = \Theta(\log N)\).
    \label{ex:independent-rv-scaling}
    \end{example}

    Theorem~\ref{thm:scaling-by-area} provides a useful tool to pin down the scaling of the size of the Pareto set.
    Due to the relatively quick decay of the additional terms in Equation~\eqref{eq:truncated-copula-convergence}, we find that scaling estimates using the region \(R_N\) are quite accurate even for modest \(N\).
    However, its applicability is limited, as it requires that we either have an analytic expression for the copula or are otherwise able to estimate the copula to precision \(1 / N\).
    In particular, we are not able to prove any bounds for the DIB frontier, which is the case \(U = H(Z)\), and \(V = I(Z; Y)\).
    We suspect that for most realistic datasets, including points on the DIB plane, that \(\ell(N) = \operatorname{polylog}(N)\), which implies that the scaling of the Pareto set is likewise \(\Theta(\operatorname{polylog}(N))\).
    Since we are interested in the large \(N\) behavior, we are hopeful that more general results can be found through the study of extreme-value copulas, which we leave for future work.


\section{Auxiliary Functions}
    \label{app:auxiliary-functions}
    In this appendix, we provide the pseudocode for the important auxiliary functions used in Algorithms \ref{alg:mapper} and \ref{alg:robust-mapper}. 
    The Pareto set data structure is a list of point structures.
    A point structure, \(p\), contains fields for both objectives \(p\mathrm{.x}\), \(p\mathrm{.y}\), and optional fields for storing the uncertainties \(p\mathrm{.dx}\), \(p\mathrm{.dy}\) and clustering function \(p\mathrm{.f}\). 
    As a list, the Pareto set \(P\) for point \(p\), and index \(i\), also supports the functions \(\textsc{size}(P)\) returning the number of elements in \(P\), \(\textsc{insert}(p, i, P)\) for inserting point \(p\) at index \(i\), and \(\textsc{remove}(i, P)\) for removing the entry at index \(i\).
    Additionally, since the Pareto set \(P\) is maintained in sorted order by its first index, we can find the correct index at which to insert a new point in logarithmic time:
    for a point \(p\) and Pareto set \(P\), this is written \(\textsc{find\_index}(p.x, P)\) in the pseudocode of Algorithms \ref{alg:is-pareto}--\ref{alg:pareto-distance}.

    \begin{algorithm}[H]
        \caption{Check if a point is Pareto optimal        }
        \label{alg:is-pareto}
        \textit{Input}: Point on objective plane \(p\), and Pareto Set \(P\)
        
        \textit{Output}: \textsc{true} if and only if \(p\) is Pareto optimal in \(P\)
        \begin{algorithmic}[1]
            \Procedure{is\_pareto}{$p, P$}
                \State \(i = \textsc{find\_index}(p\mathrm{.x}, P)\) \Comment{Return correct value to insert \(p\) in \(P\)}

                \Return{\(\textsc{size}(P) = 0\) \textbf{ or } \(i = \textsc{size}(P) \textbf{ or } P\mathrm{[i + 1].y} < p\mathrm{.y}\)}
            \EndProcedure
        \end{algorithmic}
    \end{algorithm}
\vspace{-6pt}
    \begin{algorithm}[H]
        \caption{Add point to Pareto Set        }
        \label{alg:pareto-add}
        \textit{Input}: Point on objective plane \(p\), and Pareto Set \(P\)
        
        \textit{Output}: Updated Pareto Set \(P\)
        \begin{algorithmic}[1]
            \Procedure{pareto\_add}{$p, P$}
                \If{\(\textsc{is\_pareto}(p, P)\)} \Comment{Insert only if Pareto optimal}
                    \State \(i = \textsc{find\_index}(p.x, P)\)

                    \State \(P \leftarrow \textsc{insert}(p, i, P)\) \Comment{Insert point into correct location}

                    \While{\(i < \textsc{size}(P)\) \textbf{ and } \(p\mathrm{.y} > P\mathrm{[i + 1].y}\)} \Comment{Remove dominated points}
                        \State \(\textsc{remove}(i + 1, P)\)
                        \State \(i = i + 1\)
                    \EndWhile
                \EndIf
                \Return{P}
            \EndProcedure
        \end{algorithmic}
    \end{algorithm}
\vspace{-6pt}
    \begin{algorithm}[H]
        \caption{Calculate distance to Pareto frontier        }
        \label{alg:pareto-distance}
         \textit{Input}: Point on objective plane \(p\), and Pareto Set \(P\)
         
        \textit{Output}: Distance to Pareto frontier (defined to be zero if Pareto optimal)
        \begin{algorithmic}[1]
            \Procedure{pareto\_distance}{$p, P$}
                \If{\(\textsc{is\_pareto}(p, P)\)} 
                    \Return{0} \Comment{Distance defined to be zero if point is Pareto optimal}
                \EndIf

                \State \(i = \textsc{find\_index}(p.x, P)\)
                \State \(d = P\mathrm{[i].y} - p\mathrm{.y}\) \Comment{Check top boundary}

                \While{\(P\mathrm{[i].x} - p\mathrm{.x} < d\)}
                    \If{\(i + 1 < \textsc{size}(P)\) \textbf{ and } \(P\mathrm{[i].y} > p\mathrm{.y}\)} 
                        \State \(q = \textsc{point}(x = P\mathrm{[i].x}, y = P\mathrm{[i+1].y})\)
                        \State \(d = \textsc{minimum}(\textsc{distance}(p, q), d)\) \Comment{Check corners}
                    \Else
                        \State \(d = \textsc{minimum}(P\mathrm{[i].x} - p\mathrm{.x}, d)\) \Comment{Check right boundary}
                    \EndIf
                \EndWhile

                \Return{\(d\)}
            \EndProcedure
        \end{algorithmic}
    \end{algorithm}


\section{The Symmetric Pareto Mapper}
    \label{app:symmetric-mapper}
    In this appendix, we consider one way that Algorithm~\ref{alg:mapper} can be modified to accommodate an additional structure in the dataset.
    The full pseudocode is provided in Algorithm~\ref{alg:symmetric-mapper} with the key difference occurring on line 12.
    This modification amounts to a redefining of the compressed variable \(Z = (f(X_1), f(X_2))\).
    We would like to discover an encoding \(f\) that trades off the entropy of the encoding with the ability to predict \(Y\) from \((f(X_1), f(X_2))\).
    This corresponds to the following graphical model:
    \[
    \begin{tikzcd}
    Z_1 & X_1 \arrow[l, "f"'] \arrow[r] & Y \\
    Z_2 & X_2 \arrow[l, "f"] \arrow[ru] &  
    \end{tikzcd}
    \]

    \begin{figure}[H]
   
\centering 
  {\captionsetup{position=bottom,justification=centering}
        \begin{subfigure}[b]{\textwidth}
          \centering
            \includegraphics[width=10cm]{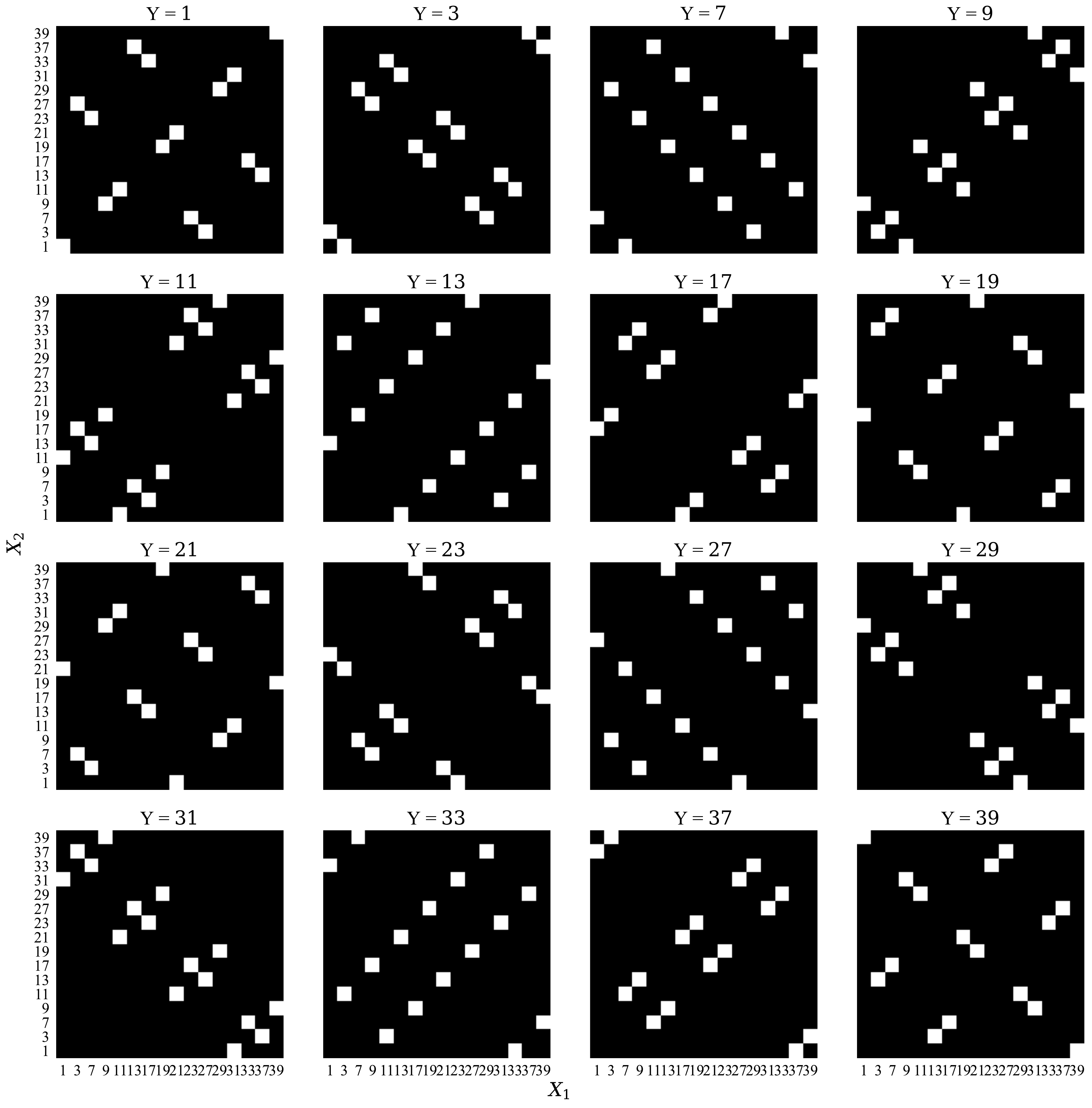}
            \caption{}
        \end{subfigure}%
        ~ 
        \\
        \begin{subfigure}[b]{\textwidth}
           \centering
            \includegraphics[width=10cm]{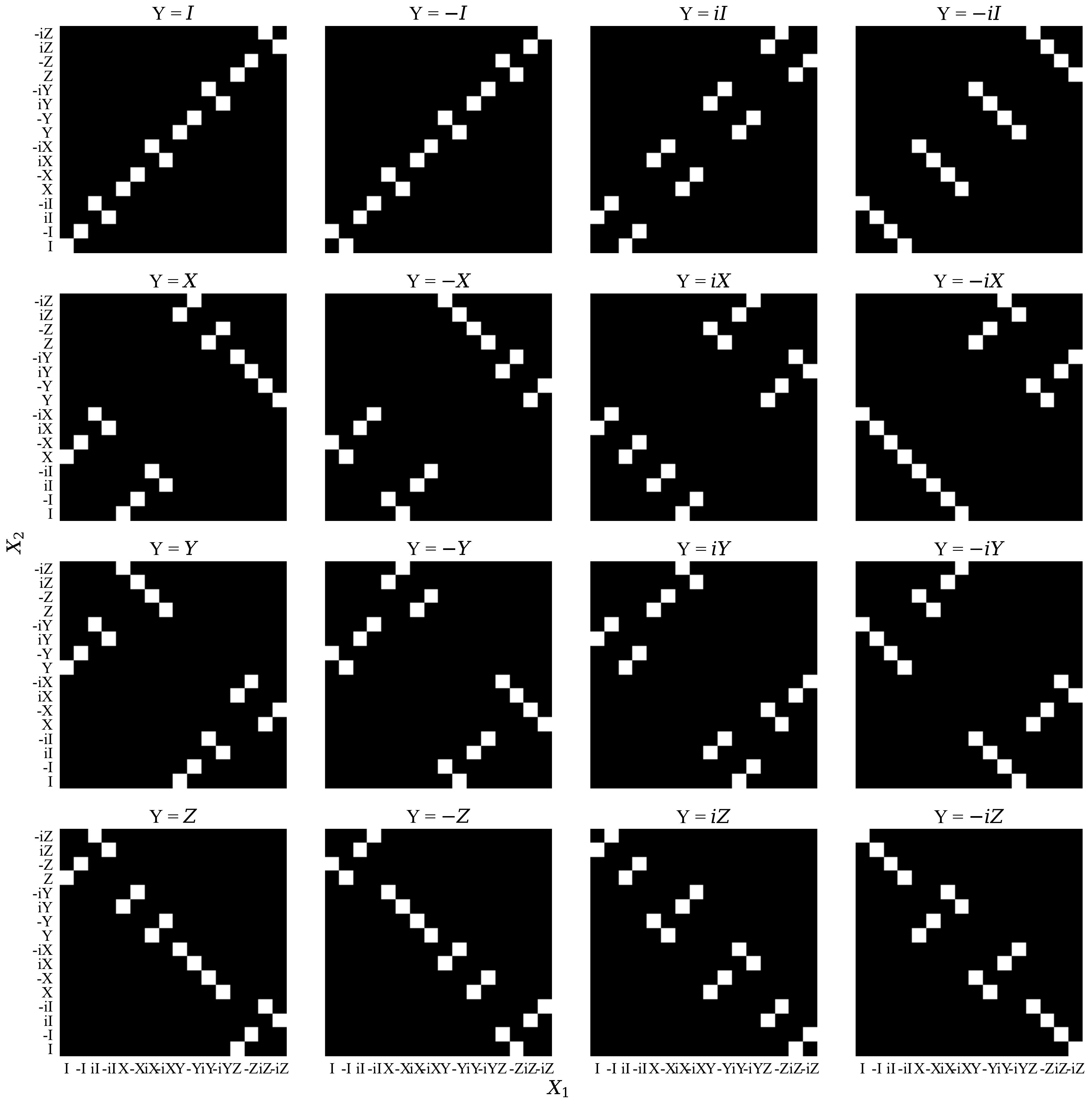}
            \caption{}
        \end{subfigure}}
        \caption{Joint distribution \(p_{X_1, X_2; Y}\) for the (\textbf{a}) \(\left(\mathbb{Z}/40\mathbb{Z}\right)^\times\) group and (\textbf{b}) the Pauli group.}
        \label{fig:group-distribution}
    \end{figure} 

    \begin{algorithm}[H]
    \caption{Symmetric Pareto Mapper   }
    \label{alg:symmetric-mapper}
      \textit{Input}: Joint distribution \(A, B, C \sim p_{ABC}\), and search parameter $\varepsilon$
      
    \textit{Output}: Approximate Pareto frontier \(P\)
    \begin{algorithmic}[1]

    \Procedure{symmetric\_pareto\_mapper}{$p_{ABC}, \varepsilon$}
        \State \textbf{Pareto set} \(P = \emptyset\) \Comment{Initialize maintained Pareto set}
        \State \textbf{Queue} \(Q = \emptyset\) \Comment{Initialize search queue}

        \State \textbf{Point} \(p = (\mathrm{x=} -\operatorname{H}(p_{X_1 X_2}) / 2, \mathrm{y=} \operatorname{I}(p_{X_1 X_2 ; Y}), \mathrm{f=} \operatorname{id})\) \Comment{Evaluate trivial clustering}
        \State \(P \leftarrow  \textsc{insert}(p, P)\) 

        \State \(Q \leftarrow  \textsc{enqueue}(\operatorname{id}, Q) \) \Comment{Start with identity clustering \(\operatorname{id}: [n] \rightarrow [n]\) where \(n = |X|\)}
        
        \While{\(Q\) is not \(\emptyset\)}
            \State \(f = \textsc{dequeue}(Q)\) 
            \State \(n = |\operatorname{range}(f)|\)
            \For{ \(0 < i < j < n \) } \Comment{Loop over all pairs of output clusters of \(f\)}
                \State \(f' = c_{i,j} \circ f\) \Comment{Merge clusters \(i, j\) output \(f\)}
                \State \textbf{Point} \(p = \mathrm{Point}(\mathrm{x=} -\operatorname{H}(p_{f'(X_1) f'(X_2)}) / 2, \mathrm{y=} \operatorname{I}(p_{f'(X_1) f'(X_2); Y}), \mathrm{f=} f')\) 
                \State \(d = \textsc{pareto\_distance}(p, P)\)

                \State \(P \leftarrow \textsc{pareto\_add}(p, P)\) \Comment{Keep track of point and clustering in Pareto set}

                \State{\textbf{with} probability \(e^{-d/\varepsilon}\),  \(Q \leftarrow \textsc{enqueue}(f', Q)\)} \\
            \EndFor
        \EndWhile 
        \Return{\(P\)}
    \EndProcedure

    \end{algorithmic}
    \end{algorithm}

\end{document}